\documentclass[conference]{IEEEtran}
\IEEEoverridecommandlockouts
\usepackage{cite}
\usepackage{amsmath,amssymb,amsfonts}
\usepackage{algorithmic}
\usepackage{graphicx}
\usepackage{textcomp}
\usepackage{xcolor}
\def\BibTeX{{\rm B\kern-.05em{\sc i\kern-.025em b}\kern-.08em
    T\kern-.1667em\lower.7ex\hbox{E}\kern-.125emX}}

\usepackage{subfigure}
\usepackage{hyperref}
\usepackage{enumitem}
\usepackage{amsfonts}
\usepackage{amssymb}
\usepackage{amsthm}
\newtheorem{definition}{Definition}[section] 
\newtheorem{theorem}{Theorem}[section]

\newtheorem{lemma}{Lemma}[section]
\usepackage{multirow}
\usepackage{dsfont}
\usepackage{algorithm2e}
\usepackage{booktabs}

\usepackage{tikz}
\usetikzlibrary{bayesnet}

\usepackage{todonotes}

\begin{document}

\title{An Intersectional Definition of Fairness\\
\thanks{This work was performed under the following financial assistance award: 60NANB18D227 from
U.S. Department of Commerce, National Institute of
Standards and Technology.
}
}


\author{\IEEEauthorblockN{James R. Foulds, Rashidul Islam, Kamrun Naher Keya, Shimei Pan}
\IEEEauthorblockA{\textit{Department of Information Systems} \\
\textit{University of Maryland, Baltimore County, USA}\\
\{jfoulds, islam.rashidul, kkeya1, shimei\}@umbc.edu}
}

\maketitle

\begin{abstract}
  	We propose definitions of fairness in machine learning and artificial intelligence systems that are informed by the framework of intersectionality, a critical lens arising from the Humanities literature which analyzes how interlocking systems of power and oppression affect individuals along overlapping dimensions including gender, race, sexual orientation, class, and disability. We show that our criteria behave sensibly for any subset of the set of protected attributes, and we prove economic, privacy, and generalization guarantees. We provide a learning algorithm which respects our intersectional fairness criteria. Case studies on census data and the COMPAS criminal recidivism dataset demonstrate the utility of our methods.
\end{abstract}



\section{Introduction}

The increasing impact of artificial intelligence and machine learning technologies on many facets of life, from commonplace movie recommendations to consequential criminal justice sentencing decisions, has prompted concerns that these systems may behave in an unfair or discriminatory manner \cite{barocas2016big,executive2016big,noble2018algorithms}.  A number of studies have subsequently demonstrated that bias and fairness issues in AI are both harmful and pervasive \cite{angwin2016machine,buolamwini2018gender,bolukbasi2016man}.  The AI community has responded by developing a broad array of mathematical formulations of fairness and learning algorithms which aim to satisfy them \cite{dwork2012fairness,hardt2016equality,berk2017convex,zhao2017men}.  Fairness, however, is not a purely technical construct, having social, political, philosophical and legal facets \cite{campolo2017ai}.  At this juncture, the necessity has become clear for interdisciplinary analyses of fairness in AI and its relationship to society, to civil rights, and to the social goals which are to be achieved by mathematical fairness definitions, which have not always been made explicit \cite{mitchell2018prediction}.

In particular, it is important to connect fairness and bias in algorithms to the broader context of fairness and bias in society, which has long been the concern of civil rights and feminist scholars and activists \cite{noble2018algorithms, keyes2019mulching}.  In this work, we address the specific challenges of fairness in AI that are motivated by \textbf{intersectionality}, an analytical lens from the third-wave feminist movement which emphasizes that civil rights and feminism should be considered simultaneously rather than separately \cite{crenshaw1989demarginalizing}.  We propose \textbf{intersectional AI fairness criteria} and perform a \emph{comprehensive, interdisciplinary analysis} of their relation to the concerns of diverse fields including the \textbf{humanities}, \textbf{law}, \textbf{privacy}, \textbf{economics}, and \textbf{statistical machine learning}.  Our contributions include:
\begin{enumerate}
    \item A critical analysis of the consequences of intersectionality in the particular context of fairness for AI,
    \item Three novel fairness metrics: \emph{differential fairness (DF)} which aims to uphold intersectional fairness for AI and machine learning systems, \emph{DF bias amplification}, a slightly more politically conservative fairness definition which measures the bias specifically introduced by an algorithm, and \emph{differential fairness with confounders} which can alter outcome distributions \emph{(DFC)}, 
    \item Proofs of the desirable intersectionality, privacy, economic, and generalization properties of our metrics,
    \item A learning algorithm which enforces our criteria, and
    \item Case studies on census and criminal recidivism data which demonstrate our methods' practicality and their benefits versus the subgroup fairness criterion of \cite{kearns2018preventing}.
\end{enumerate}


\section{Intersectionality and Fairness in AI}

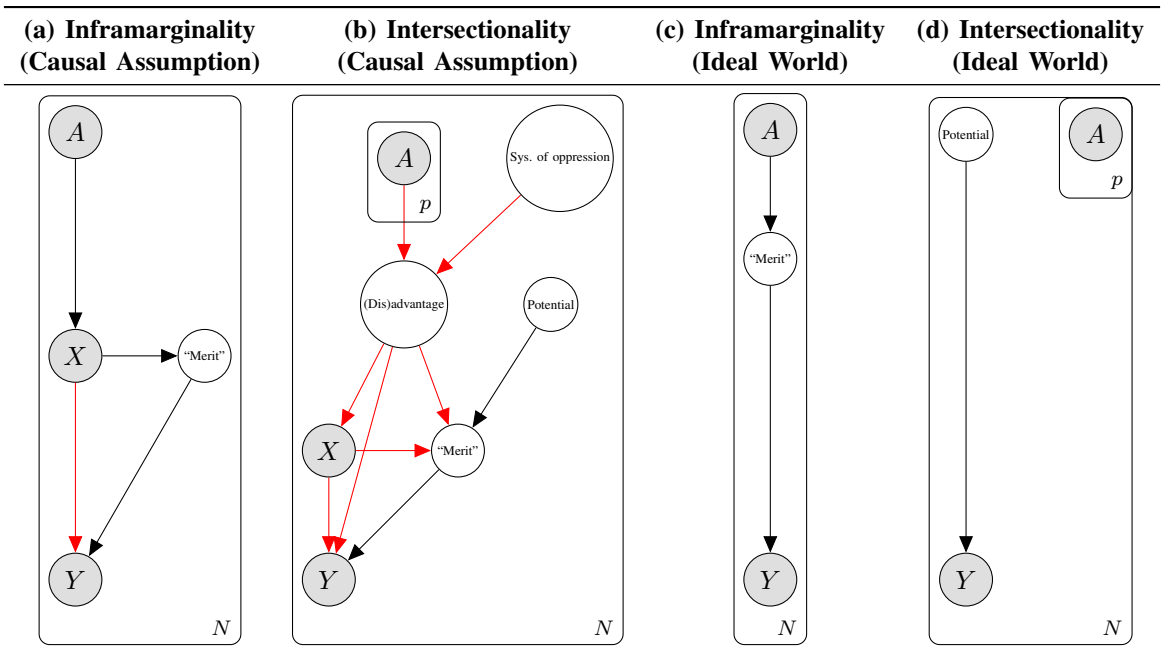
\begin{figure*}[t]
\centering
\begin{tabular}{cccc}
	\toprule
	\textbf{(a) Inframarginality} & \textbf{(b) Intersectionality} & \textbf{(c) Inframarginality} & \textbf{(d) Intersectionality} \\
	\textbf{(Causal Assumption)} & \textbf{(Causal Assumption)} & \textbf{(Ideal World)} &  \textbf{(Ideal World)} \\ 
	\midrule
	\begin{tikzpicture}[y=2.26cm]
		\node[obs] (Y) {$Y$} ;
		\node[obs, above=of Y] (X) {$X$} ;
		\node[latent, right=of X] (M) {\tiny ``Merit''};
		\node[obs, above=of X] (A) {$A$} ;
		
		\edge {A} {X}
		\edge {X} {M}
		\edge[red] {X} {Y}
		\edge {M} {Y}
		
		\plate {infraPeople} {(A)(X)(M)(Y)} {$N$}
	\end{tikzpicture}
	 &
	 
	 \begin{tikzpicture}
		 \node[obs] (Y) {$Y$} ;
		 \node[obs, above=of Y] (X) {$X$} ;
		 \node[latent, right=of X] (M) {\tiny ``Merit''} ;
		 \node[latent, above=of X, xshift = 1cm] (P) {\tiny (Dis)advantage} ;
		 \node[latent, right=of P] (U) {\tiny Potential} ;
		 \node[obs, above=of P] (A) {$A$} ;
		 \node[latent, right=of A] (S) {\tiny Sys. of oppression} ;

		 \edge[red] {A,S} {P}
		 \edge[red] {P} {M}
		 \edge[] {U} {M}
		 \edge[red] {P} {X}
		 \edge[red] {X} {M}
		 \edge[red] {X} {Y}
		 \edge {M} {Y}
		 \edge[red] {P} {Y}
		 
		 \plate {interPeople} {(A)(X)(M)(Y)(P)(U)(S)} {$N$}
		 \plate {interAtts} {(A)} {$p$}
	 \end{tikzpicture}
	&
	\begin{tikzpicture}[y=3.55cm]
		\node[obs] (Y) {$Y$} ;
		\node[latent, above=of Y] (M) {\tiny ``Merit''};
		\node[obs, above=of X] (A) {$A$} ;
		
		\edge {A} {M}
		\edge {M} {Y}
		
		\plate {infraIdealPeople} {(A)(M)(Y)} {$N$}	
	\end{tikzpicture}
	& 
	\begin{tikzpicture}[y=5.2cm]
	\node[obs] (Y) {$Y$} ;
	\node[latent, above=of Y] (U) {\tiny Potential} ;
	\node[obs, right=of U] (A) {$A$} ;	
	
	\edge {U} {Y}
	
	\plate {interPeople} {(A)(Y)(U)} {$N$}
	\plate {interAtts} {(A)} {$p$}
	\end{tikzpicture} \\
	\bottomrule
\end{tabular}
\caption{Implicit causal assumptions (a,b) and values-driven ideal world scenarios (c,d) for inframarginality and intersectionality notions of fairness. Here, $A$ denotes protected attributes, $X$ observed attributes, $Y$ outcomes, $N$ individuals, $p$ number of protected attributes.  Red arrows denote potentially unfair causal pathways, which are removed to obtain the ideal world scenarios (c,d). The above summarizes broad strands of research; individual works may differ.\label{fig:causalAssump}}
\end{figure*}

We begin with an introduction to intersectionality and an analysis of its relationship to fairness in an artificial intelligence and machine learning context.  Intersectionality is a lens for examining societal unfairness which originally arose from the observation that sexism and racism have intertwined effects, in that the harm done to Black women by these two phenomena is more than the sum of the parts \cite{crenshaw1989demarginalizing,truth1851aint}.  The notion of intersectionality was later extended to include overlapping injustices along more general axes \cite{collins2002black}.  In its general form, intersectionality emphasizes that systems of oppression built into society lead to \emph{systematic disadvantages along intersecting dimensions}, which include not only gender, but also race, nationality, sexual orientation, disability status, and socioeconomic class \cite{collective1977black,collins2002black,crenshaw1989demarginalizing,hooks1981ain,lorde1984age,truth1851aint}.
These systems are interlocking in their effects on individuals at \emph{each intersection of the affected dimensions}.  

The term \emph{intersectionality} was introduced by  Kimberl\'e Crenshaw in the 1980's \cite{crenshaw1989demarginalizing} and popularized in the 1990's, e.g. by Patricia Hill Collins \cite{collins2002black}, although the ideas are much older \cite{collective1977black,truth1851aint}.  In the context of machine learning and fairness, intersectionality was recently considered by \cite{buolamwini2018gender}, who studied the impact of the intersection of gender and skin color on computer vision performance, and by \cite{kearns2018preventing,hebert-johnson2018multicalibration}, who aimed to protect certain subgroups in order to prevent ``fairness gerrymandering.'' From a humanities perspective, \cite{noble2018algorithms} critiqued the behavior of the Google search engine with an intersectional lens, by examining the search results for terms relating to women, people of color, and their intersections, e.g.  ``Black girls.''

Intersectionality has implications for AI fairness beyond the use of multiple \emph{protected attributes}.  Many fairness definitions aim (implicitly or otherwise) to uphold the principle of \textbf{infra-marginality}, which states that differences between protected groups in the distributions of ``merit'' or ``risk'' (e.g. the probability of carrying contraband at a policy stop) should be taken into account when determining whether bias has occurred \cite{simoiu2017problem}.
A closely related argument is that parity of outcomes between groups is at odds with accuracy \cite{dwork2012fairness,hardt2016equality}.  
Intersectionality theory provides a counterpoint: these differences in risk/merit, while acknowledged, are frequently due to systemic structural disadvantages such as racism, sexism, inter-generational poverty, the school-to-prison pipeline, mass incarceration, and the prison-industrial complex \cite{collective1977black,crenshaw1989demarginalizing,davis2011prisons,hooks1981ain,wald2003defining}.
Systems of oppression can lead individuals to perform below their potential, for instance by reducing available cognitive bandwidth \cite{verschelden2017bandwidth}, or by increasing the probability of incarceration \cite{davis2011prisons,alexander2012new}.
In short, the \emph{infra-marginality} principle makes the implicit assumption that society is a fair, level playing field, and thus differences in ``merit'' or ``risk'' between groups in data and predictive algorithms are often to be considered legitimate.  In contrast, \emph{intersectionality} theory posits that these \textbf{distributions of merit and risk are often influenced by unfair societal processes} (see Figure \ref{fig:causalAssump}).

As an example of a scenario affected by unfair processes, consider the task of predicting prospective students' academic performance for use in college admissions decisions.  As discussed in detail by \cite{verschelden2017bandwidth}, and references therein, individuals belonging to marginalized and non-majority groups are disproportionately impacted by challenges of poverty and racism (in its structural, overt, and covert forms), including chronic stress, access to healthcare, under-treatment of mental illness, micro-aggressions, stereotype threat, disidentification with academics, and belongingness uncertainty.  Similarly, LGBT and especially transgender, non-binary, and gender non-conforming students disproportionately suffer bullying, discrimination, self-harm, and the burden of concealing their identities.
These challenges are often further magnified at the intersection of affected groups.  A survey of 6,450 transgender and gender non-conforming individuals found that the most serious discrimination was experienced by people of color, especially Black respondents \cite{grant2011injustice}.
Verschelden explains the impact of these challenges as a tax on the ``cognitive bandwidth'' of non-majority students, which in turn affects their academic performance.  She states that the evidence is clear 
\begin{quote}
	\emph{
	``...that racism (and classism, homophobia, etc.) has made people physically, mentally, and spiritually ill and dampened their chance at a fair shot at higher education (and at life and living).''}
\end{quote}
A classifier trained to predict students' academic performance from historical data hence aims to emulate outcomes that were substantially affected by unfair factors \cite{barocas2016big}.
An \emph{accurate predictor} for a student's GPA may therefore not correspond to a \emph{fair decision-making procedure} \cite{berk2017fairness}.
We can resolve this apparent conflict if we are careful to distinguish between the \emph{statistical problem} of classification, and the \emph{economic problem} of the assignment of outcomes (e.g. admission decisions) to individuals based on classification. 
Viewing the classifier's task as a policy question, it becomes clear that high accuracy need not be the primary goal of the system, especially when we consider that ``accuracy'' is measured on unfair data.\footnote{Amazon recently abandoned a classifier for job candidate selection which was found to be gender biased \cite{dastin2018amazon}.  We speculate that this was likely due to similar issues.}

In Figure \ref{fig:causalAssump} we summarize the causal assumptions regarding society and data, and the idealized ``perfect world'' scenarios implicit in the two approaches to fairness.  Inframarginality \emph{(a)} emphasizes that the distribution over relevant attributes $X$ varies across protected groups $A$, which leads to potential differences in so-called ``merit'' or ``risk'' between groups, typically presumed to correspond to latent ability and thus ``deservedness'' of outcomes $Y$ \cite{simoiu2017problem}.  Intersectionality \emph{(b)} emphasizes that we must also account for systems of oppression which lead to (dis)advantage at the intersection of multiple protected groups, impacting all aspects of the system including the ability of individuals to succeed (``merit'') to their potential, had they not been impacted by (dis)advantage \cite{crenshaw1989demarginalizing}.  In the ideal world that an algorithmic (or other) intervention aims to achieve, inframarginality-based fairness desires that individual ``merit'' is the sole determiner of outcomes \emph{(c)} \cite{simoiu2017problem,hardt2016equality}, which can lead to disparity between groups \cite{dwork2012fairness}.  In ideal intersectional fairness \emph{(d)}, since ability to succeed is affected by unfair processes, it is desired that this unfairness is corrected and individuals achieve their true potential \cite{verschelden2017bandwidth}.  Assuming potential does not substantially differ across protected groups, this implies that parity between groups is typically desirable.\footnote{Disparity could still be desirable if there are legitimate confounders which depend on protected groups, e.g. choice of department that individuals apply to in college admissions.  We address this scenario in Section \ref{sec:confounders}.}

In light of the above, we argue that an \emph{intersectional} definition of fairness in AI should satisfy the following criteria:
\begin{enumerate}[label=\Alph*]
	\item Multiple protected attributes should be considered. \label{criterion:multiAtt}
	\item \textbf{All} of the intersecting values of the protected attributes, e.g. \emph{Black women}, should be protected by the definition. \label{criterion:intersect}
	\item We should still also ensure that protection is provided on individual protected attribute values, e.g. \emph{women}. \label{criterion:individualAtt} 
	\item The definition should protect minority groups, who are often particularly affected by discrimination in society. \label{criterion:minority}
	\item The definition should ensure that systematic differences between the protected groups, assumed to be due to structural oppression, are rectified, rather than codified. \label{criterion:oppress}
\end{enumerate}
These desiderata do not uniquely specify a fairness definition, but they provide a set of guidelines to which legal, political, and contextual considerations can then be applied to determine an appropriate fairness measure for a particular task. 

\begin{figure}[t]     
		\centerline{\includegraphics[width=0.45\textwidth]{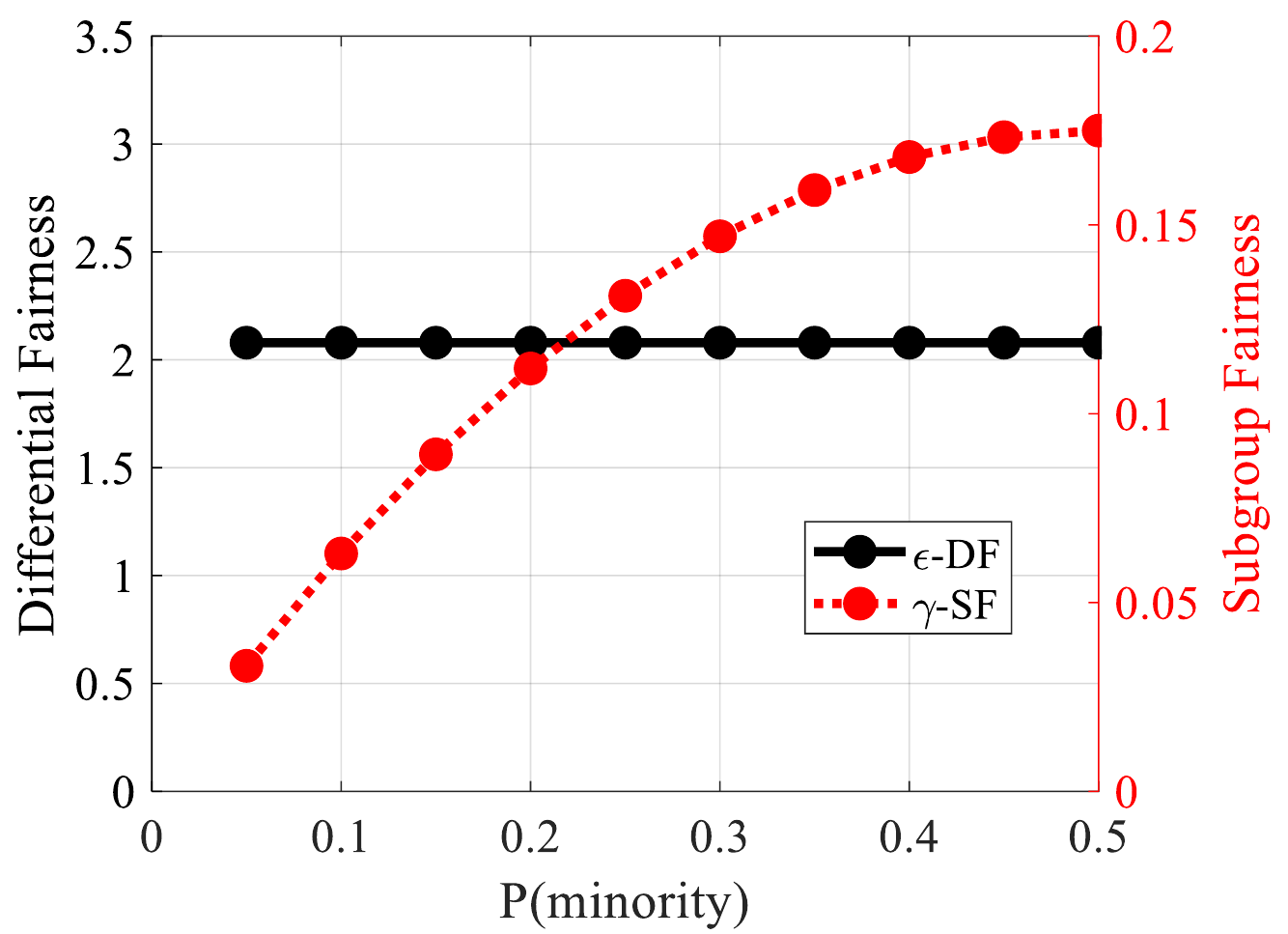}}
		\caption{\small 
		Toy example: probability of the ``positive'' class is 0.8 for a majority group, 0.1 for a minority group, varying $P(\mbox{minority})$.
		}
		\label{fig:toyExample} 
\end{figure}

\section{Existing Fairness Definitions}
\label{sec:intersectionality}
We now consider existing fairness definitions and their relation to the aforementioned criteria (see the Appendix for further discussion of related work).  Relevant fairness definitions aim to detect and prevent discriminatory (or other) bias with respect to a \emph{set of protected attributes}, such as gender, race, and disability status.  Given criterion \ref{criterion:multiAtt}, we focus on multi-attribute definitions.  The two dominant multi-attribute approaches in the literature are \emph{subgroup fairness} \cite{kearns2018preventing} and \emph{multicalibration} \cite{hebert-johnson2018multicalibration}.

We adapt the notation of \cite{kifer2014pufferfish} to all definitions in this paper.  
Suppose $M(\mathbf{x})$ is a (possibly randomized) mechanism which takes an instance $\mathbf{x} \in \chi$ and produces an outcome $y$ for the corresponding individual, $S_1, \ldots, S_p$ are discrete-valued protected attributes, $A = S_1 \times S_2 \times \ldots \times S_p$, and $\theta$ is the distribution which generates $\mathbf{x}$. For example, the mechanism $M(\mathbf{x})$ could be a deep learning model for a lending decision, $A$ could be the applicant's possible gender and race, and $\theta$ the joint distribution of credit scores and protected attributes.  The protected attributes are included in the attribute vector $\mathbf{x}$, although $M(\mathbf{x})$ is free to disregard them (e.g. if this is disallowed). The setting is illustrated in Figure \ref{fig:DFsetting}.

\begin{figure*}[t]
    \centering
	\includegraphics[width=\textwidth]{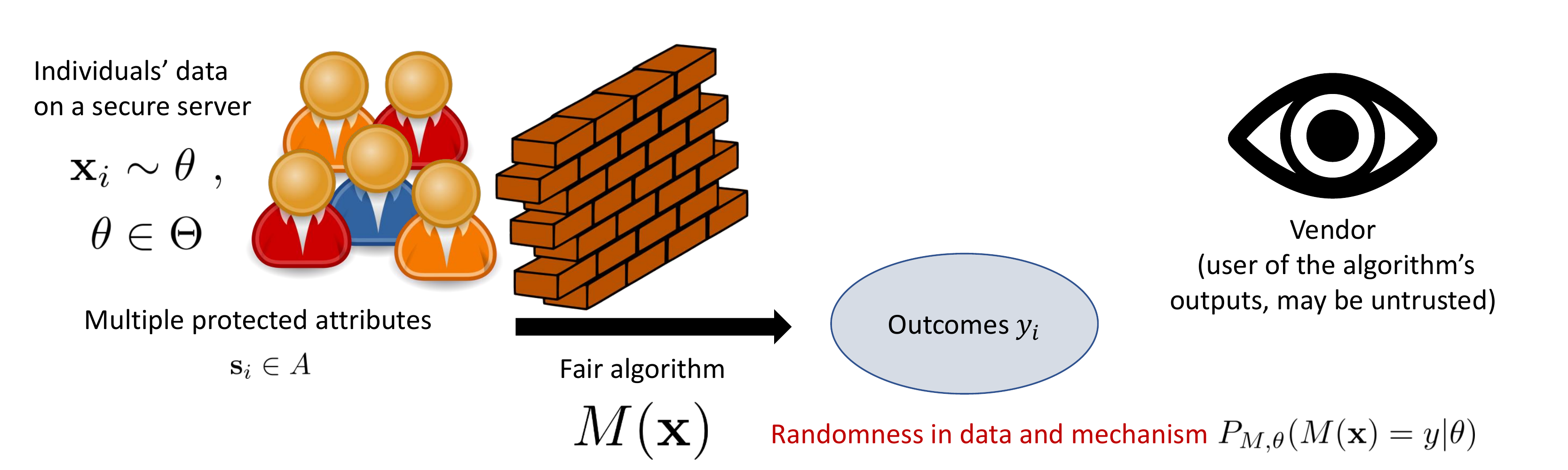}
	\caption{Diagram of the setting for the proposed differential fairness criterion. \label{fig:DFsetting}}
\end{figure*}

\begin{definition} (Statistical Parity Subgroup Fairness \cite{kearns2018preventing})
    Let $\mathcal{G}$ be a collection of protected group indicators $g: A \rightarrow \{0, 1\}$, where $g(\mathbf{s}) = 1$ designates that an individual with protected attributes $\mathbf{s}$ is in group $g$.  Assume that the classification mechanism $M(\mathbf{x})$ is binary, i.e. $y \in \{0,1\}$.
    
    Then $M(\mathbf{x})$ is $\gamma$-\emph{statistical parity subgroup fair} with respect to $\theta$ and $\mathcal{G}$ if for every $g \in \mathcal{G}$,
    \begin{align}
       & | P_{M,\theta}(M(\mathbf{x}) = 1) - P_{M,\theta}(M(\mathbf{x}) = 1 | g(\mathbf{s}) = 1) | \nonumber \\
       & \times P_\theta(g(\mathbf{s}) = 1) \leq \gamma \mbox{ .} \label{eqn:SF}
    \end{align}
\end{definition}
Note that $\gamma \in [0,1]$, smaller is better.  The first term penalizes a difference between the probability of the \emph{positive} class label for group $g$, and the population average of this probability.  The term $P_\theta(g(\mathbf{s}) = 1)$ weights the penalty by the size of group $g$ as a proportion of the population.  \emph{Statistical parity subgroup fairness} (\emph{SF}) is a multi-attribute definition satisfying criterion \ref{criterion:multiAtt}.  To satisfy \ref{criterion:intersect} and \ref{criterion:individualAtt}, $\mathcal{G}$ can be \emph{all} intersectional subgroups (e.g. \emph{Black women}) and top-level groups (e.g. \emph{men}).  The first term in Equation \ref{eqn:SF}, which encourages similar outcomes between groups, enforces criterion \ref{criterion:oppress}.

\begin{figure}[t]
		\centerline{\includegraphics[width=0.45\textwidth]{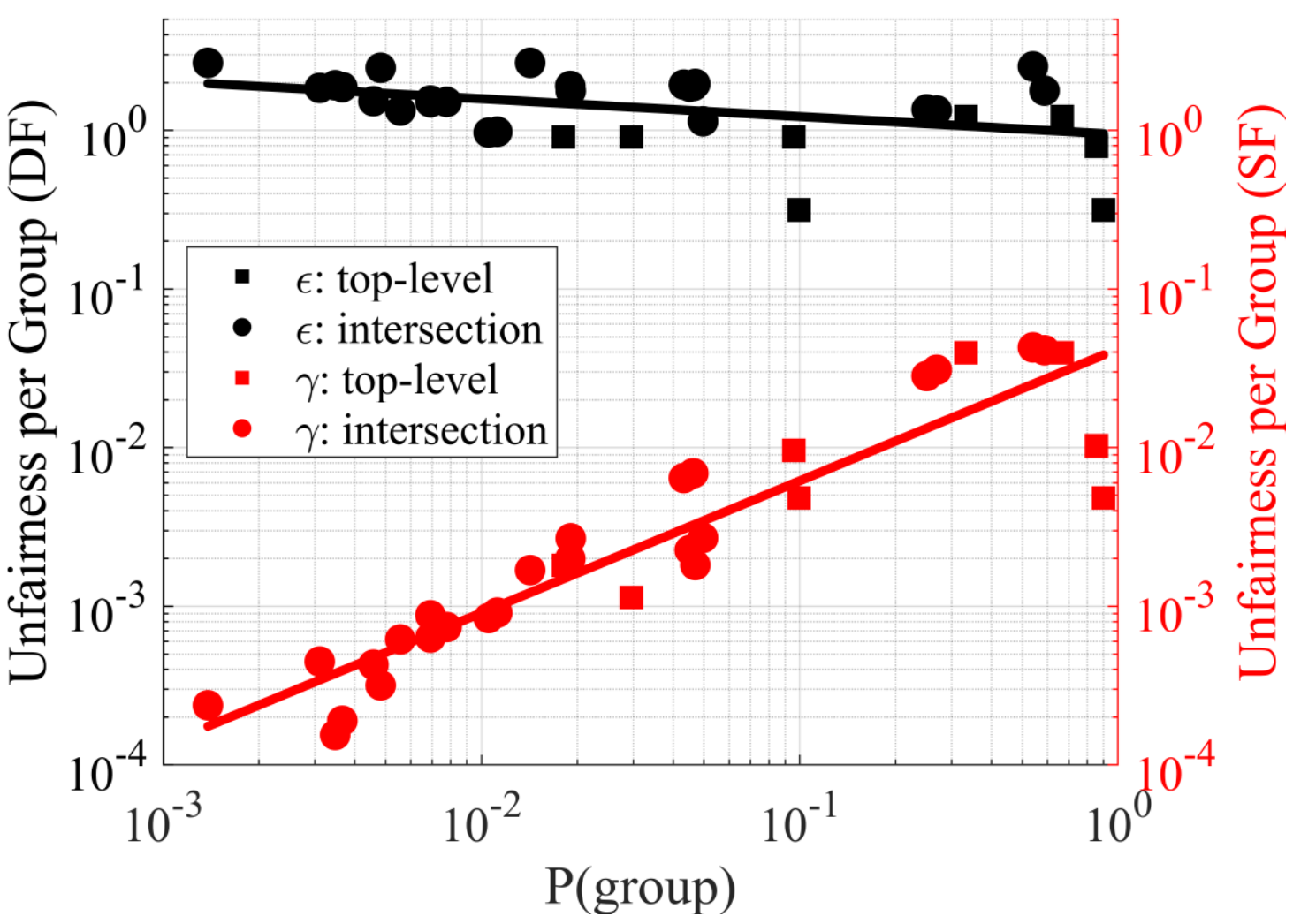}}
		\caption{\small ``Per-group'' $\gamma$-SF and our proposed $\epsilon$-DF, vs probability (i.e. size) of groups, Adult dataset. Circles: intersectional subgroups (e.g. Black women of USA). Squares: top-level groups (e.g. men). }
		\label{fig:compareWithSFdata}
\end{figure}

From an intersectional perspective, one concern with SF is that it does not satisfy criterion \ref{criterion:minority}, the protection of minority groups.  The term $P_\theta(g(\mathbf{s}) = 1)$ weights the ``per-group (un)fairness'' for each group $g$, i.e. Equation \ref{eqn:SF} applied to $g$ alone,  by its proportion of the population, thereby \emph{specifically downweighting the consideration of minorities}.  In Figure \ref{fig:toyExample}, we show an example where varying the size of a minority group $P(\mbox{minority})$ drastically alters $\gamma$-subgroup fairness, which finds that \textbf{a rather extreme scenario is more acceptable when the minority group is small}.  Our proposed criterion, $\epsilon$-DF (introduced in Section \ref{sec:DF}), is constant in $P(\mbox{minority})$.

Figure \ref{fig:compareWithSFdata} reports ``per-group'' $\gamma$'s on the UCI Adult census dataset, i.e. Equation \ref{eqn:SF} applied separately to each group, empirically seen have an increasing relationship with $P(\mbox{group})$.  The final $\gamma$-SF is determined by the worst case of the per-group $\gamma$'s. \emph{A small minority group thereby will most likely not directly affect} $\gamma$-SF, since the downweighting makes it unlikely to be the ``most unfair'' group.  

Kearns et al. \cite{kearns2018preventing} justify the use of the $P_\theta(g(\mathbf{s}) = 1)$ term via statistical considerations, as it is useful to prove generalization guarantees to extrapolate from empirical estimates of $\gamma$ (see Section \ref{sec:generalization}).  From a different ethical perspective, total utilitarianism, increasing the utility (i.e. reducing unfairness) for a large group of individuals at the expense of smaller groups could also be justified by the increase in the total utility of the population.  The problem with total utilitarianism, of course, is that it admits a scenario where many people possess low utility. We do not intend to dismiss SF as a valid notion of fairness.  Our claim here, rather, is simply that due to its treatment of minority groups, SF does not fully encapsulate the principles of fairness advocated by intersectional feminist scholars and activists \cite{collins2002black,crenshaw1989demarginalizing,hooks1981ain,lorde1984age,truth1851aint}.

Other candidate multi-attribute fairness definitions include \emph{false positive subgroup fairness} \cite{kearns2018preventing} and \emph{multicalibration} \cite{hebert-johnson2018multicalibration}.  These definitions are similar to SF, but they concern false-positive rates and calibration of prediction probabilities, respectively.  Since they focus on reliability of estimation rather than allocation of outcomes, they do not directly address criterion \ref{criterion:oppress}, and so are weaker definitions from a civil rights/feminist perspective.  This does not preclude their use for intersectional fairness scenarios in which harms are caused by incorrect predictions, rather than unfair outcome assignments; indeed, this is the type of approach \cite{buolamwini2018gender} take for studying intersectional fairness in computer vision applications.  Nevertheless, we will not consider them further here.

\section{Differential Fairness (DF) Measure}
\label{sec:DF}

\begin{figure*}[t]
\hspace{-1cm}
	\includegraphics[width=\textwidth]{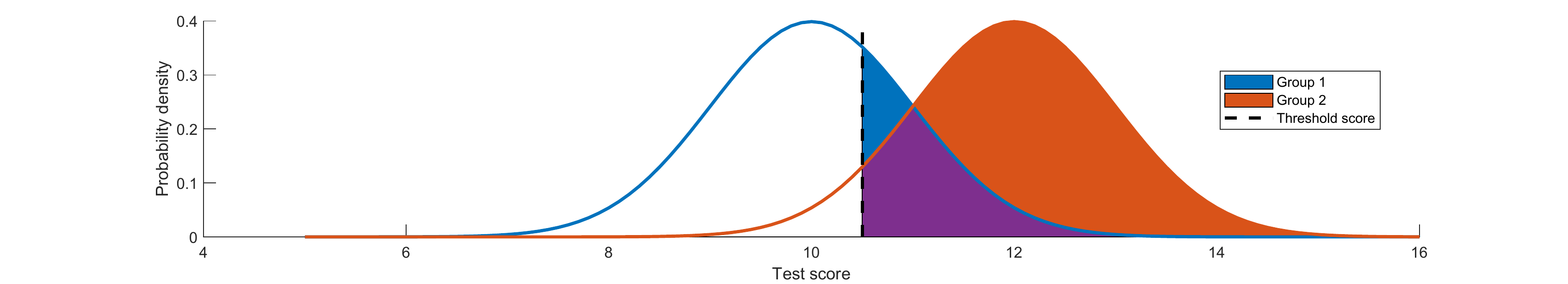}
	\begin{minipage}{0.49\textwidth}
		\centering
		\small 
		\begin{tabular}{@{}llll@{}}
			\toprule
			\multicolumn{4}{l}{Probability of Hiring Outcome Given Group}\\
			\midrule
			&    & \multicolumn{2}{c}{Group}\\
			\cmidrule{3-4}
			&    & 1 & 2\\
			\midrule
			\multirow{2}{*}{Outcome} &   yes     & 0.3085   & 0.9332 \\
			&   no      & 0.6915  & 0.0668\\
			\bottomrule
		\end{tabular}
	\end{minipage}	
	\begin{minipage}{0.49\textwidth}
		\centering
		\small 
		\begin{tabular}{@{}lllrl@{}}
			\toprule
			\multicolumn{4}{l}{Log Ratios of Probabilities} & \\
			\cmidrule{1-4}
			$y$ & $\mathbf{s}_i$ & $\mathbf{s}_j$ & \multicolumn{2}{l}{$ \log \frac{P_{M, \theta}(M(\mathbf{x}) = y|\mathbf{s}_i, \theta)}{P_{M, \theta}(M(\mathbf{x}) = y|\mathbf{s}_j, \theta)}$}\\
			\midrule
			\multirow{2}{*}{no} & 1 & 2 & 2.337 & \\
			& 2 & 1 & -2.337 & \\
			\multirow{2}{*}{yes} & 1 & 2 & -1.107 & \\
			& 2 & 1 & 1.107 & \\
			\bottomrule
		\end{tabular}
	\end{minipage}
	\caption{Worked example of differential fairness from Section \ref{sec:workedExample}. The calculations above show that $\epsilon = 2.337$.\label{fig:workedExample}}	
\end{figure*}


We now introduce our proposed fairness measures which satisfy our intersectionality criteria from Section \ref{sec:intersectionality}.  Note that there are multiple conceivable fairness definitions which satisfy these criteria.  For example, SF could be adapted to address criterion \ref{criterion:minority} by simply dropping the $P_\theta(g(\mathbf{s}) = 1)$ term,  at the loss of its associated generalization guarantees.  We instead select an alternative formulation, which is similar to this approach in spirit, but which has additional beneficial properties from a societal perspective regarding the \emph{law}, \emph{privacy}, and \emph{economics}, as we shall discuss below.  Our formalism has a particularly elegant \emph{intersectionality property}, in that Criterion \ref{criterion:individualAtt} (protecting higher-level groups) follows automatically from Criterion \ref{criterion:intersect} (protecting intersectional subgroups).

We motivate our criteria from a legal perspective.  Consider the 80\% rule, established in the Code of Federal Regulations \cite{eeoc1966guidelines} as a guideline for establishing disparate impact in violation of anti-discrimination laws such as Title VII of the Civil Rights Act of 1964.  The 80\% rule states that there is legal evidence of adverse impact if the ratio of probabilities of a particular favorable outcome, taken between a disadvantaged and an advantaged group, is less than 0.8:
\begin{align}
\small
    P(M(\mathbf{x}) = 1| \mbox{group A}) / P(M(\mathbf{x}) = 1| \mbox{group B}) < 0.8 \mbox{ .} 
\end{align}

Our first proposed criterion, which we call \textbf{differential fairness (DF)}, extends the 80\% rule to protect multi-dimensional intersectional categories, with respect to multiple output values.  We similarly restrict ratios of outcome probabilities between groups, but instead of using a predetermined fairness threshold at 80\%, we measure fairness on a sliding scale that can be interpreted similarly to that of \emph{differential privacy}, a definition of privacy for data-driven algorithms \cite{dwork2006calibrating}.
Differential fairness measures the \textbf{fairness cost} of mechanism $M(\mathbf{x})$ with a parameter $\epsilon$.
\begin{definition}
	A mechanism $M(\mathbf{x})$ is $\epsilon$-\emph{differentially fair (DF)} with respect to $(A, \Theta)$ if for all $\theta \in \Theta$ with $\mathbf{x} \sim \theta$, and $y \in \mbox{Range}(M)$,
	\begin{equation}
	e^{-\epsilon} \leq \frac{P_{M, \theta}(M(\mathbf{x}) = y|\mathbf{s}_i, \theta)}{P_{M, \theta}(M(\mathbf{x}) = y|\mathbf{s}_j, \theta)}\leq e^\epsilon \mbox{ ,} \label{eqn:DF}
	\end{equation}
	for all   $(\mathbf{s}_i, \mathbf{s}_j) \in A \times A$ where $P(\mathbf{s}_i|\theta) > 0$, $P(\mathbf{s}_j|\theta) > 0$.
	\label{def:DF}
\end{definition}
In Equation \ref{eqn:DF}, $\mathbf{s}_i$, $\mathbf{s}_j \in A$ are tuples of \emph{all} protected attribute values, e.g. gender, race, and nationality, and $\Theta$ is a set of distributions $\theta$ which could plausibly generate each instance $\mathbf{x}$.\footnote{The possibility of multiple $\theta \in \Theta$ is valuable from a privacy perspective, where $\Theta$ is the set of \emph{possible beliefs} that an adversary may have about the data, and is motivated by the work of \cite{kifer2014pufferfish}.    Continuous protected attributes are also possible, in which case sums are replaced by integrals in our proofs.} For example, $\Theta$ could be the set of Gaussian distributions over credit scores per value of the protected attributes, with mean and standard deviation in a certain range.

This is an intuitive \textbf{intersectional definition of fairness}: \emph{regardless of the combination of protected attributes, the probabilities of the outcomes will be similar}, as measured by the ratios versus other possible values of those variables, for small values of $\epsilon$.  For example, the probability of being given a loan would be similar regardless of a protected group's intersecting combination of gender, race, and nationality, marginalizing over the remaining attributes in $\mathbf{x}$.  If the probabilities are always equal, then $\epsilon = 0$, otherwise $\epsilon > 0$.  We have arrived at our criterion based on the 80\% rule, but it can also be derived as a special case of \emph{pufferfish} \cite{kifer2014pufferfish}, a generalization of differential privacy \cite{dwork2013algorithmic} which uses a variation of Equation \ref{eqn:DF} to hide the values of an arbitrary set of secrets.
\begin{definition}
	A mechanism $M(\mathbf{x})$ is $\epsilon$-\emph{pufferfish private} \cite{kifer2014pufferfish} in a framework $(S, Q, \Theta)$ if for all $\theta \in \Theta$ with $\mathbf{x} \sim \theta$, for all secret pairs $(\mathbf{s}_i,\mathbf{s}_j) \in Q$ and $y \in \mbox{Range}(M)$,
	\begin{equation}
	e^{-\epsilon} \leq \frac{P_{M, \theta}(M(\mathbf{x}) = y|\mathbf{s}_i, \theta)}{P_{M, \theta}(M(\mathbf{x}) = y|\mathbf{s}_j, \theta)}\leq e^\epsilon \mbox{ ,} \label{def:pufferfish}
	\end{equation}
	when $\mathbf{s}_i$ and $\mathbf{s}_j$ are such that  $P(\mathbf{s}_i|\theta) > 0$, $P(\mathbf{s}_j|\theta) > 0$.
\end{definition}
Differential fairness adapts pufferfish to the task of defining algorithmic fairness, by selecting a set of protected attributes as the secrets, and ensuring that the values of these attributes are indistinguishable.  Thus, differential fairness provides a closely related privacy guarantee to differential privacy.

If $P_{M,\theta}$ is unknown, it can be estimated using the empirical distribution, or via a probabilistic model of the data.  Assuming discrete outcomes, $P_{Data}(y|\mathbf{s}) = \frac{N_{y,\mathbf{s}}}{N_{\mathbf{s}}}$, where $N_{y,\mathbf{s}}$ and $N_{\mathbf{s}}$ are empirical counts of their subscripted values in the dataset $D$.  \textbf{Empirical differential fairness (EDF)} corresponds to verifying that for any $y$, $\mathbf{s}_i$, $\mathbf{s}_j$, we have
\begin{equation}
e^{-\epsilon} \leq \frac{N_{y,\mathbf{s}_i}}{N_{\mathbf{s}_i}}\frac{N_{\mathbf{s}_j}}{N_{y,\mathbf{s}_j}}\leq e^\epsilon  \mbox{ ,} \label{eqn:discreteDataEDF}
\end{equation}
Alternatively, if we estimate $\epsilon$-$DF$ via the posterior predictive distribution of a Dirichlet-multinomial model, the criterion for any $y$, $\mathbf{s}_i$, $\mathbf{s}_j$ becomes
\begin{equation}
e^{-\epsilon} \leq \frac{N_{y,\mathbf{s}_i} + \alpha}{N_{\mathbf{s}_i}  + |\mathcal{Y}|\alpha}\frac{N_{\mathbf{s}_j} + |\mathcal{Y}|\alpha}{N_{y,\mathbf{s}_j} + \alpha}\leq e^\epsilon \mbox{ ,} \label{eqn:smoothedFairness}
\end{equation}
where scalar $\alpha$ is each entry of the parameter of a symmetric Dirichlet prior with concentration parameter $|\mathcal{Y}|\alpha$, $\mathcal{Y} = \mbox{Range}(M)$. We refer to this as \textbf{smoothed EDF}. 

Note that EDF and smoothed EDF methods can sometimes be unstable in extreme cases when nearly all instances are assigned to the same class. To address this issue, instead of using empirical hard counts per group $N_{y,s}$, we can also use \emph{soft counts} for (smoothed) EDF, based on a probabilistic classifier's predicted $P(y|\mathbf{x})$, as follows:
\begin{equation}
e^{-\epsilon} \leq \frac{\sum_{\mathbf{x} \in D: A = \mathbf{s}_i} P(y|\mathbf{x}) + \alpha}{N_{\mathbf{s}_i}  + |\mathcal{Y}|\alpha}\frac{N_{\mathbf{s}_j} + |\mathcal{Y}|\alpha}{\sum_{\mathbf{x}  \in D: A = \mathbf{s}_j} P(y|\mathbf{x}) + \alpha}\leq e^\epsilon \mbox{ .} 
\label{eqn:smoothedFairnessSoft}
\end{equation}

\section{DF Bias Amplification Measure}
We can adapt DF to measure fairness in data, i.e.  outcomes assigned by a black-box algorithm or social process, by using (a model of) the data's generative process as the mechanism.
\begin{definition}
	A labeled dataset $D = \{(\mathbf{x}_1,y_1),\ldots, (\mathbf{x}_N,y_N)\}$ is $\epsilon$-\emph{differentially fair (DF)} in $A$ with respect to model $P_{Model}(\mathbf{x},y)$ if mechanism $M(\mathbf{x}) = y \sim P_{Model}(y|\mathbf{x})$ is $\epsilon$-\emph{differentially fair} with respect to $(A, \{P_{Model}(\mathbf{x})\})$, for $P_{Model}$ trained on the dataset.
	\label{def:DFdata}
\end{definition}


Similarly to differential privacy, differences $\epsilon_2 - \epsilon_1$ between two mechanisms $M_2(\mathbf{x})$ and $M_1(\mathbf{x})$ are meaningful (for fixed $A$ and $\Theta$, and for tightly computed minimum values of $\epsilon$), and measure the additional ``fairness cost'' of using one mechanism instead of the other.  When $\epsilon_1$ is the differential fairness of a labeled dataset and $\epsilon_2$ is the differential fairness of a classifier measured on the same dataset, $\epsilon_2 - \epsilon_1$ is a measure of the extent to which the classifier increases the unfairness over the original data, a phenomenon that \cite{zhao2017men} refer to as \emph{bias amplification}.
\begin{definition}
	A mechanism $M(\mathbf{x})$ satisfies $(\epsilon_2 - \epsilon_1)$-\emph{DF bias amplification} with respect to $(A, \Theta, D, \mathcal{M})$ if it is $\epsilon_2$-DF and $D$ is a labeled dataset which is $\epsilon_1$-DF with respect to model $\mathcal{M}$.
	\label{def:biasAmplification}
\end{definition}

Politically speaking, $\epsilon$-$DF$ is a relatively progressive notion of fairness which we have motivated based on intersectionality (\emph{disparities in societal outcomes are largely due to systems of oppression}), and which is reminiscent of demographic parity \cite{dwork2012fairness}.  On the other hand, $(\epsilon_2 - \epsilon_1)$-$DF$ \textbf{bias amplification} is a more \textbf{politically conservative fairness metric} which does not seek to correct unfairness in the original dataset (i.e. it relaxes criterion \ref{criterion:oppress}), in line with the principle of \textbf{infra-marginality} (\emph{a system is biased only if disparities in its behavior are worse than those in society}) \cite{simoiu2017problem}.  Informally, $\epsilon_2$-$DF$ and $(\epsilon_2 - \epsilon_1)$-$DF$ bias amplification represent ``upper and lower bounds'' on the unfairness of the system in the case where the relative effect of structural oppression on outcomes is unknown.

\section{Illustrative Worked Examples}
\label{sec:workedExample}
A simple worked example of differential fairness is given in Figure \ref{fig:workedExample}.  In the example, given an applicant's score $x$ on a standardized test, the mechanism $M(x) = x \geq t$ approves the hiring of a job applicant if their test score $x \geq t$, with $t = 10.5$.  The scores are distributed according to $\theta$, which corresponds to the following process. The applicant's protected group is 1 or 2 with probability 0.5.  Test scores for group 1 are normally distributed $N(x; \mu_1 = 10, \sigma = 1)$, and for group 2 are distributed $N(x; \mu_2 = 12, \sigma = 1)$.  In the figure, the group-conditional densities are plotted on the top, along with the threshold for the hiring outcome being \emph{yes} (i.e. $M(x) = 1$).  Shaded areas indicate the probability of a \emph{yes} hiring decision for each group (overlap in purple).  On the bottom, the calculations show that $M(x)$ is $\epsilon$-differentially fair for $\epsilon = 2.337$.  This means that the probability ratios are bounded within the range $(e^{-\epsilon}, e^\epsilon) = (0.0966, 10.35)$, i.e. one group has around 10 times the probability of some particular hiring outcome than the other ($y = $ \emph{no}).  Under the presumption that the two groups are roughly equally capable of performing the job overall, this is clearly unsatisfactory in terms of fairness.

The \emph{intersectional} setting, in which there are multiple protected variables, is specifically addressed by differential fairness, by considering the probabilities of outcomes for each intersection of the set of protected variables.
\begin{table}[t]
	\centering
	\small
	\begin{tabular}{@{}lllll@{}}
		\toprule
		\multicolumn{5}{l}{Probability of Being Admitted to University X}\\
		\midrule
		&    & \multicolumn{2}{c}{Gender}&\\
		\cmidrule{3-4}
		&    & A & B & Overall\\
		\midrule
		\multirow{2}{*}{Race} &   1     & $\frac{81}{87} \ (0.931)$   & $\frac{234}{270} \ (0.867)$ & $\frac{315}{357} \ (0.882)$ \vspace{0.1cm}\\
		&   2      & $\frac{192}{263} \ (0.730)$  & $\frac{55}{80} \ (0.688)$ & $\frac{247}{343} \ (0.720)$\\
		\cmidrule{2-4}
		& Overall & $\frac{273}{350} \ (0.780)$ & $\frac{289}{350} \ (0.826)$ &\\
		\bottomrule
	\end{tabular}
	\caption{Intersectional example: Simpson's paradox. \label{tab:simpsonsExample}}	
\end{table}
%
We illustrate this setting with an example on admissions of prospective students to a particular University X.  In the scenario, the protected attributes are gender and race, and the mechanism is the  admissions process, with a binary outcome.  
Our data, shown in Table \ref{tab:simpsonsExample}, is adapted from a real-world scenario involving treatments for kidney stones, often used to demonstrate Simpson's paradox \cite{charig1986comparison, julious1994confounding}. Here, the ``paradox'' is that for race 1, individuals of gender A are more likely to be admitted than those of gender B, and for race 2, those of gender A are also more likely to be admitted than those of gender B, yet counter-intuitively, gender B is more likely to be admitted overall.  

Since the admissions process is a black box, we model it using Equation \ref{eqn:discreteDataEDF}, empirical differential fairness (EDF).  By calculating the log probability ratios of $(Gender, Race)$ pairs from Table \ref{tab:simpsonsExample}, as well as for the pairs of probabilities for the declined admission outcome ($1 - P(\mbox{admit})$), and plugging them into Equation \ref{eqn:discreteDataEDF}, we see that the mechanism is $\epsilon=1.511$-DF with $A = Gender \times Race$.  By calculating $\epsilon$ using the admission probabilities in the \emph{Overall} row ($Gender$) and the \emph{Overall} column ($Race$), we find that $\epsilon = 0.2329$ for $A = Gender$, and $\epsilon = 0.8667$ for $A = Race$.  We will prove in Theorem \ref{thm:fullIntersectionality}  that $\epsilon$ with $A = Gender \times Race$ is an upper bound on
$\epsilon$-DF for $A = Gender$ and for $A = Race$.  Thus, even with a ``Simpson's reversal'' differential (un)fairness will not increase after summing out a protected attribute.

\section{Dealing with Confounder Variables}
\label{sec:confounders}
\begin{figure}
\centering
\begin{tikzpicture}
\node[obs] (Y) {$Y$} ;
\node[latent, above=of Y] (U) {\tiny Potential} ;
\node[latent, right=of U] (C) {\tiny Confounders} ; 
\node[obs, above=of C] (A) {$A$} ;	

\edge {U} {Y}
\edge {A} {C}
\edge {C} {Y}

\plate {interPeople} {(A)(Y)(U)(C)} {$N$}
\plate {interAtts} {(A)} {$D$}
\end{tikzpicture}
\caption{Ideal-world intersectional fairness but with counfounder variables present.  Disparity in overall outcomes between protected groups may occur.\label{fig:confounders}}
\end{figure}
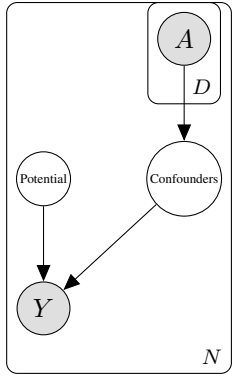

As we have seen, differential fairness can be used to measure the inequity between the outcome probabilities for the protected groups and their intersections at different levels of measurement granularity, although it does not determine whether the inequities were due to systemic factors and/or discrimination. In the case study above, a confounding variable which could explain the Simpson's reversal is the decision of the prospective student on whether to apply to University X. 
The $\epsilon$-DF criterion is appropriate when the differences are believed to be due to systems of oppression, as posited by intersectionality theory, and such confounder variables are not present.  With confounders, parity in outcomes between intersectional protected groups, which $\epsilon$-DF rewards, may no longer be desirable (see Figure \ref{fig:confounders}). We propose an alternative fairness definition for when known confounders are present. 
\begin{definition}
	Let $\theta \in \Theta$ be distributions over $(\mathbf{x}, c)$, where $c \in C$ are confounder variables. A mechanism $M(\mathbf{x})$ is $\epsilon$-\emph{differentially fair with confounders (DFC)} with respect to $(A, \Theta, C)$, if for all $c \in C$, $M(\mathbf{x})$ is $\epsilon$-DF with respect to $(A, \Theta_{|c})$, where  $\Theta_{|c} = \{P(\mathbf{x}|\theta, c) | \theta \in \Theta\}$.
	\label{def:DFC}
\end{definition}
In the university admissions case, Definition \ref{def:DFC} penalizes disparity in admissions at the department level, and the most unfair department determines the overall unfairness $\epsilon$-DFC.
\begin{theorem}
	\label{thm:confounders}
	Let $M$ be an $\epsilon$-DFC mechanism in $(A, \Theta,C)$, Then $M$ is $\epsilon$-\emph{differentially fair} in $(A, \Theta)$.
\end{theorem}
From Theorem \ref{thm:confounders}, if we protect differential fairness per department, we obtain differential fairness and its corresponding theoretical economic and privacy guarantees in the University's overall admissions, bounded by the $\epsilon$ of the most unfair department, \emph{even in the case of a Simpson's reversal}.  A proof is given in the Appendix.  If confounder variables are latent, we can attempt to infer them probabilistically in order to apply DFC.  Alternatively, ($\epsilon_2 - \epsilon_1$)-DF bias amplification can still be used to study the impact of an algorithm on fairness. 

\section{Properties of Differential Fairness}
\label{sec:DFproperties}

We now discuss the theoretical properties of our definitions.
\subsection{Differential Fairness and Intersectionality}
Differential fairness explicitly encodes protection of intersectional groups (criterion \ref{criterion:intersect}).  For DF, we prove that this automatically implies fairness for \emph{each of the protected attributes individually} (criterion \ref{criterion:individualAtt}), and indeed, \emph{any subset} of the protected attributes.
For example, if a loan approval mechanism $M(\mathbf{x})$ is $\epsilon$-DF in $A = $ \emph{gender} $\times$ \emph{race} $\times$ \emph{nationality}, it is also $\epsilon$-DF in, e.g., $A = $ \emph{gender} by itself, or $A = $ \emph{gender} $\times$ \emph{nationality}.  In other words, by ensuring fairness at the intersection of gender, race, and nationality under our criterion, we also ensure the same degree of fairness between genders overall, and between gender/nationality pairs overall, and so on. In the above, $\epsilon$ is a worst case, and DF may also hold for lower values of $\epsilon$. 

\begin{lemma} \label{lem:rewriteDF} (Proof given in the Appendix.)
	The $\epsilon$-DF criterion can be rewritten as: for any $\theta \in \Theta$, $y \in \mbox{Range}(M)$, 
\begin{align}
 &\log \max_{\mathbf{s} \in A: P(\mathbf{s}|\theta) > 0} P_{M, \theta}(M(\mathbf{x}) = y|\mathbf{s}, \theta) \nonumber \\
 &- \log \min_{s  \in A: P(\mathbf{s}|\theta) > 0} P_{M, \theta}(M(\mathbf{x}) = y|\mathbf{s}, \theta) \leq \epsilon \mbox{ .}
\end{align}
\end{lemma}

\begin{theorem} \textbf{(Intersectionality Property)}
	\label{thm:fullIntersectionality}
	Let $M$ be an \\ $\epsilon$-\emph{differentially fair} mechanism in $(A, \Theta)$, $A =S_1 \times S_2 \times\ldots \times S_p$, and let $D = S_a \times \ldots \times S_k$ be the Cartesian product of a nonempty proper subset of the protected attributes included in $A$. Then $M$ is $\epsilon$-\emph{differentially fair} in $(D, \Theta)$.
\end{theorem}
\begin{proof}
Define $E = S_1\times \ldots\times S_{a-1}\times S_{a + 1}\ldots \times S_{k-1}\times S_{k + 1}\times \ldots \times S_p$, the Cartesian product of the protected attributes included in $A$ but not in $D$.  Then for any  $\theta \in \Theta$, $y \in \mbox{Range}(M)$, 

    \begin{align}
    &\log \max_{\mathbf{s} \in D: P(\mathbf{s}|\theta) > 0} P_{M, \theta}(M(\mathbf{x}) = y|D = s, \theta) \nonumber \\ 
    =& \log  \max_{\mathbf{s} \in D: P(\mathbf{s}|\theta) > 0} \sum_{e \in E} P_{M, \theta}(M(\mathbf{x}) = y|E=e, \mathbf{s}, \theta) P_{\theta}(E=e|\mathbf{s},\theta) \nonumber \\
    \leq & \log  \max_{\mathbf{s} \in D: P(\mathbf{s}|\theta) > 0} \sum_{e \in E} \max_{e' \in E: P_{\theta}(E=e'|\mathbf{s},\theta) > 0}  \nonumber \\
    &\ \ \ \ \ \ \ \ \ \ \ \big ( P_{M, \theta}(M(\mathbf{x}) = y|E=e', \mathbf{s}, \theta) \big ) \times P_{\theta}(E=e|\mathbf{s},\theta)  \nonumber \\
    = & \log  \max_{\mathbf{s} \in D: P(\mathbf{s}|\theta) > 0} \max_{e' \in E: P_{\theta}(E=e'|\mathbf{s},\theta) > 0}P_{M, \theta}(M(\mathbf{x}) = y|E=e', \mathbf{s}, \theta) \nonumber \\
    =& \log \max_{\mathbf{s}' \in A: P(\mathbf{s}'|\theta) > 0} P_{M, \theta}(M(\mathbf{x}) = y|\mathbf{s}', \theta) \nonumber
    \end{align}
    By a similar argument, $\log \min_{\mathbf{s}  \in D: P(\mathbf{s}|\theta) > 0} P_{M, \theta}(M(\mathbf{x}) = y|D = \mathbf{s}, \theta)  \geq \log \min_{\mathbf{s}'  \in A: P(\mathbf{s}'|\theta) > 0} P_{M, \theta}(M(\mathbf{x}) = y|\mathbf{s}', \theta)$.  
    Applying Lemma \ref{lem:rewriteDF}, we hence bound $\epsilon$ in  $(D, \Theta)$ as
    \begin{align}
    &\log \max_{\mathbf{s} \in D: P(\mathbf{s}|\theta) > 0} P_{M, \theta}(M(\mathbf{x}) = y|D = \mathbf{s}, \theta) \nonumber \\ 
    &- \log \min_{\mathbf{s} \in D: P(\mathbf{s}|\theta) > 0} P_{M, \theta}(M(\mathbf{x}) = y|D = \mathbf{s}, \theta) \nonumber \\
    \leq & \log \max_{\mathbf{s}' \in A: P(\mathbf{s}'|\theta) > 0} P_{M, \theta}(M(\mathbf{x}) = y|\mathbf{s}', \theta) \nonumber \\
    &- \log \min_{\mathbf{s}'  \in A: P(\mathbf{s}'|\theta) > 0} P_{M, \theta}(M(\mathbf{x}) = y|\mathbf{s}', \theta) \leq \ \epsilon \mbox{ .}
    \end{align}
    
\end{proof}
This property is philosophically concordant with intersectionality, which emphasizes empathy with all overlapping marginalized groups. However, its benefits are mainly practical: in principle, one could protect all higher-level groups in SF by specifying $\sum_{j=1}^p \binom{p}{j} K^j$ binary indicator protected groups, where $K$ is the number of values per protected attribute.  This quickly becomes computationally and statistically infeasible. 
For example, Figure \ref{fig:countIntersectionalitySavings} counts the number of protected groups that must be explicitly considered under the two intersectional fairness definitions, in order to respect the intersectional fairness criteria \ref{criterion:intersect} and \ref{criterion:individualAtt}.  The intersectionality property (Theorem \ref{thm:fullIntersectionality}) implies that when the the bottom-level intersectional groups are protected (blue curve), differential fairness will automatically protect all higher-level groups and subgroups (red curve).  Since subgroup fairness does not have this property, all of the groups and subgroups (red curve) must be protected explicitly with their own group indicators $g(\mathbf{s})$.  Although the number of bottom-level groups grows exponentially in the number of protected attributes, the total number of groups grows much faster, at the combinatorial rate of $\sum_{j=1}^p \binom{p}{j} K^j$.

\begin{figure}[t]
	\includegraphics[width=0.45\textwidth]{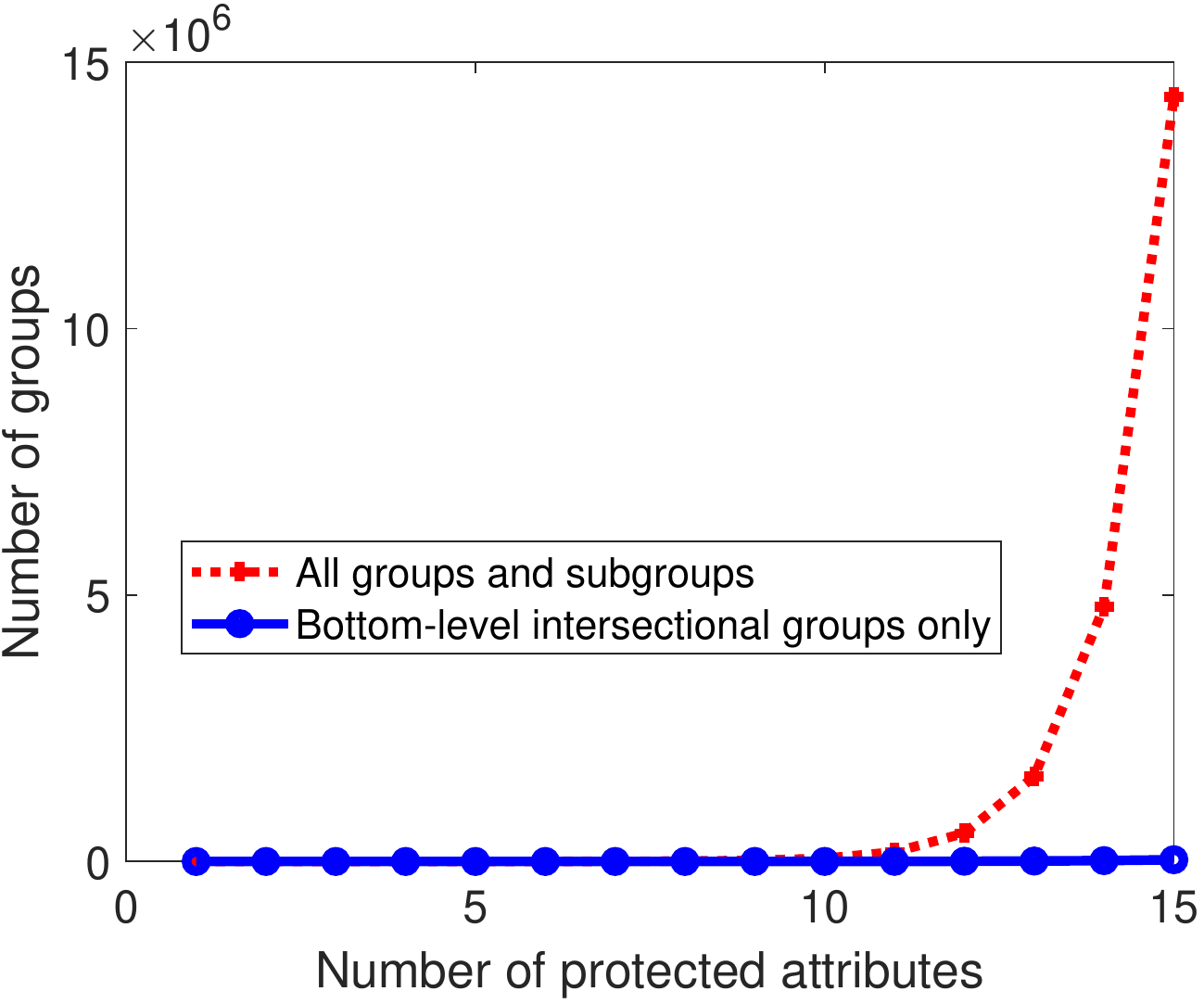}
	\caption{\label{fig:countIntersectionalitySavings} The number of groups and intersectional subgroups to protect when varying the number of protected attributes, with 2 values per protected attribute. }
\end{figure}

\subsection{Privacy Interpretation}
\label{sec:bayesPrivacy}
The differential fairness definition, and the resulting level of fairness obtained at any particular measured fairness parameter $\epsilon$, can be interpreted by viewing the definition through the lens of privacy.  Differential fairness ensures that given the outcome, \emph{an untrusted vendor/adversary can learn very little about the protected attributes of the individual}, relative to their prior beliefs, assuming their prior beliefs are in $\Theta$:
\begin{equation}
e^{-\epsilon}\frac{P(\mathbf{s}_i|\theta)}{P(\mathbf{s}_j|\theta)} \leq \frac{P(\mathbf{s}_i|M(\mathbf{x}) = y, \theta)}{P(\mathbf{s}_j|M(\mathbf{x}) = y, \theta)}\leq e^\epsilon \frac{P(\mathbf{s}_i|\theta)}{P(\mathbf{s}_j|\theta)} \mbox{ .} \label{eqn:bayesPrivacy}
\end{equation}
E.g., if a loan is given to an individual, an adversary's Bayesian posterior beliefs about their race and gender will not be substantially changed.  Thus, the adversary will be unable to infer that ``this individual was given a loan, so they are probably white and male.''
Our definition thereby provides fairness guarantees when the user of $M(\mathbf{x})$ is untrusted, cf. \cite{dwork2012fairness}, by preventing subsequent discrimination, e.g. in retaliation to a fairness correction. 
Although DF is a \emph{population-level} definition, it provides a privacy guarantee for \emph{individuals}.
The privacy guarantee only holds if $\theta \in \Theta$, which may not always be the case.  Regardless, the value of $\epsilon$ may typically be interpreted as a privacy guarantee against a ``reasonable adversary.''  The privacy guarantee is inherited from \emph{pufferfish}, a general privacy framework which DF instantiates \cite{kifer2014pufferfish}.

\begin{table*}[t]
\centering
\resizebox{0.95\textwidth}{!}{
\begin{tabular}{llcccccc}
\hline
\multicolumn{2}{c}{\multirow{2}{*}{Models}}                                                                        & \multicolumn{3}{c}{\textbf{DF-Classifier}}                                  & \multicolumn{2}{c}{\textbf{SF-Classifier}}    & \multirow{2}{*}{\textbf{Typical Classifier}} \\
\multicolumn{2}{c}{}                                                                                               & $\epsilon_{1}=0.0$ & $\epsilon_{1}=0.2231$ & $\epsilon_{1}=\epsilon_{data}$ & $\gamma_{1}=0.0$ & $\gamma_{1}=\gamma_{data}$ &                                              \\ \hline
\multirow{3}{*}{Performance Measures}                                                              & Accuracy      & 0.811              & 0.823                 & 0.839                          & 0.835            & 0.839                      & 0.839                                        \\
                                                                                                   & F1 Score      & 0.470              & 0.520                 & 0.600                          & 0.550            & 0.590                      & 0.602                                        \\
                                                                                                   & ROC AUC       & 0.849              & 0.862                 & 0.885                          & 0.882            & 0.886                      & 0.892                                        \\ \hline
\multirow{4}{*}{\begin{tabular}[c]{@{}l@{}}Fairness Measures\\ (using soft counts)\end{tabular}}   & $\epsilon$-DF & 0.428              & \textbf{0.379}                 & 1.629                          & 1.334            & 1.590                      & 1.646                                        \\
                                                                                                   & $\gamma$-SF   & \textbf{0.006}              & 0.012                 & 0.039                          & 0.026            & 0.034                      & 0.041                                        \\
                                                                                                   & Bias Amp-DF   & -0.952             & \textbf{-1.001}                & 0.249                          & -0.046           & 0.210                      & 0.266                                        \\
                                                                                                   & Bias Amp-SF   & \textbf{-0.027}             & -0.021                & 0.006                          & -0.007           & 0.001                      & 0.008                                        \\ \hline
\multirow{4}{*}{\begin{tabular}[c]{@{}l@{}}Fairness Measures  \\ (using hard counts)\end{tabular}} & $\epsilon$-DF & \textbf{1.602}              & 1.676                 & 2.034                          & 1.843            & 1.843                      & 2.115                                        \\
                                                                                                   & $\gamma$-SF   & \textbf{0.003}              & 0.010                 & 0.034                          & 0.017            & 0.026                      & 0.040                                        \\
                                                                                                   & Bias Amp-DF   & \textbf{-0.303}             & -0.229                & 0.129                          & -0.062           & -0.062                     & 0.210                                        \\
                                                                                                   & Bias Amp-SF   & \textbf{-0.037}             & -0.030                & -0.006                         & -0.023           & -0.014                     & 0.000                                        \\ \hline
\end{tabular}
}
\textbf{\caption{\small Comparison of intersectionally fair classifiers with the typical classifier on the Adult dataset ($\epsilon_1 = 0.2231$ is the 80\% rule).}
\label{fig:Adult-dataset}}
\end{table*}

\subsection{Economic Guarantees}
We also show that differential fairness provides economic guarantees.  
An $\epsilon$-differentially fair mechanism admits a disparity in expected utility of as much as a factor of $\exp(\epsilon) \approx 1 + \epsilon$ (for small values of $\epsilon$) between pairs of protected groups with $\mathbf{s}_i \in A$, $\mathbf{s}_j \in A$, for any utility function that could be chosen.  E.g., consider a loan approval process, where the utility of being given a loan is 1, and being denied is 0.   Suppose the approval process is $\ln(3)$-differentially fair
. The process could then be three times as likely to award a loan to white men as to white women, and thus award white men three times the expected utility as white women.  The proof follows the case of differential privacy \cite{dwork2013algorithmic}.  Let $u(y): \mbox{Range}(M(\mathbf{x})) \rightarrow \mathds{R}_{\geq 0}$ be a utility function.  Then:
\begin{align}
\small 
E_{P_{M,\theta}}\big [u(y)|\mathbf{s}_i \big ] &= \int P_{M,\theta}(y|\mathbf{s}_i) u(y) dy \label{eqn:utility} \\
&\leq \int e^{\epsilon} P_{M,\theta}(y|\mathbf{s}_j) u(y) dy 
= e^{\epsilon} E_{P_{M,\theta}}\big [u(y)|\mathbf{s}_j\big ] \mbox{ .} \nonumber 
\end{align}
Similarly, for $(\epsilon_2 - \epsilon_1)$-DF bias amplification, $M(\mathbf{x})$ admits at most an $\exp(\epsilon_2 - \epsilon_1) \approx 1 + \epsilon_2 - \epsilon_1$ (for small values of $\epsilon_2 - \epsilon_1$) multiplicative increase in the disparity of expected utility between pairs of protected intersections of groups with $\mathbf{s}_i \in A$, $\mathbf{s}_j \in A$, relative to the data generating process $\mathcal{M}$.

\begin{table*}[t]
\centering
\resizebox{0.95\textwidth}{!}{
\begin{tabular}{llcccccc}
\hline
\multicolumn{2}{c}{\multirow{2}{*}{Models}}                                                                        & \multicolumn{3}{c}{\textbf{DF-Classifier}}                                  & \multicolumn{2}{c}{\textbf{SF-Classifier}}    & \multirow{2}{*}{\textbf{Typical Classifier}} \\
\multicolumn{2}{c}{}                                                                                               & $\epsilon_{1}=0.0$ & $\epsilon_{1}=0.2231$ & $\epsilon_{1}=\epsilon_{data}$ & $\gamma_{1}=0.0$ & $\gamma_{1}=\gamma_{data}$ &                                              \\ \hline
\multirow{3}{*}{Performance Measures}                                                              & Accuracy      & 0.686              & 0.684                 & 0.692                          & 0.690            & 0.697                      & 0.700                                        \\
                                                                                                   & F1 Score      & 0.633              & 0.642                 & 0.643                          & 0.622            & 0.647                      & 0.641                                        \\
                                                                                                   & ROC AUC       & 0.730              & 0.723                 & 0.734                          & 0.719            & 0.739                      & 0.734                                        \\ \hline
\multirow{4}{*}{\begin{tabular}[c]{@{}l@{}}Fairness Measures\\ (using soft counts)\end{tabular}}   & $\epsilon$-DF & \textbf{0.180}              & 0.281                 & 0.410                          & 0.404            & 0.468                      & 0.773                                        \\
                                                                                                   & $\gamma$-SF   & \textbf{0.006}              & 0.021                 & 0.033                          & 0.007            & 0.028                      & 0.035                                        \\
                                                                                                   & Bias Amp-DF   & \textbf{-0.360}             & -0.259                & -0.130                         & -0.136           & -0.072                     & 0.233                                        \\
                                                                                                   & Bias Amp-SF   & \textbf{-0.015}             & 0.000                 & 0.012                          & -0.014           & 0.007                      & 0.014                                        \\ \hline
\multirow{4}{*}{\begin{tabular}[c]{@{}l@{}}Fairness Measures  \\ (using hard counts)\end{tabular}} & $\epsilon$-DF & \textbf{0.207}              & 0.671                 & 0.884                          & 0.825            & 0.860                      & 0.897                                        \\
                                                                                                   & $\gamma$-SF   & \textbf{0.015}              & 0.045                 & 0.060                          & 0.017            & 0.048                      & 0.062                                        \\
                                                                                                   & Bias Amp-DF   & \textbf{-0.339}             & 0.125                 & 0.338                          & 0.279            & 0.314                      & 0.351                                        \\
                                                                                                   & Bias Amp-SF   & \textbf{-0.025}             & 0.005                 & 0.020                          & -0.023           & 0.008                      & 0.022                                        \\ \hline
\end{tabular}
}
\textbf{\caption{\small Comparison of intersectionally fair classifiers with the typical classifier on the COMPAS dataset  ($\epsilon_1 = 0.2231$ is the 80\% rule).}
\label{fig:COMPAS-dataset}}
\end{table*}

\subsection{Generalization Guarantees} \label{sec:generalization}
In order to ensure that an algorithm is truly fair, it is important that the fairness properties obtained on a dataset will extend to the underlying population.  Kearns et al. \cite{kearns2018preventing} proved that empirical estimates of the quantities per group which determine subgroup fairness, $P_{M,\theta}(y = 1 | g(\mathbf{s}) = 1) P_\theta(g(\mathbf{s}) = 1)$, will be similar to their true values, with enough data relative to the VC dimension of the classification model's concept class $\mathcal{H}$.  We state their result below.
\begin{theorem} \label{thm:SFgeneralization}
	\textbf{\cite{kearns2018preventing}'s Theorem 2.11 (SP Uniform Convergence).} Fix a class of functions $\mathcal{H}$ and a class of group indicators $\mathcal{G}$. For any distribution $P$, let $S \sim P^m$ be a dataset consisting of $m$ examples $(\mathbf{x}_i, y_i)$ sampled i.i.d. from $P$. Then for any $0 < \delta < 1$, with probability $1 - \delta$, for every $h \in \mathcal{H}$ and $g \in \mathcal{G}$, we have:
	\begin{align}
	&|P(y = 1 | g(\mathbf{s}) = 1,h) P(g(\mathbf{s}) = 1) \nonumber \\
	&- P_{S}(y = 1 | g(\mathbf{s}) = 1, h) P_S(g(\mathbf{s}) = 1)| \nonumber \\
	&\leq \tilde{O} \Big ( \sqrt{\frac{(\mbox{VCDIM}(\mathcal{H}) + \mbox{VCDIM}(\mathcal{G}))\log m + \log(1/\delta)}{m}} \Big ) \mbox{ .}
	\end{align}	
\end{theorem}
Here, $\tilde{O}$ hides logarithmic factors, and $P_S$ is the empirical distribution from the $S$ samples.
It is natural to ask whether a similar result holds for differential fairness.  As \cite{kearns2018preventing} note, the SF definition was chosen for statistical reasons, revealed in the above equation: the $P_\theta(g(\mathbf{s}) = 1)$ term in SF arises naturally in their generalization bound.  For DF, we specifically avoid this term due to its impact on minority groups, and must instead bound $P_{M,\theta}(y|\mathbf{s})$ per group $\mathbf{s}$.  For this case, we prove the following generalization guarantee.
\begin{theorem}
	 Fix a class of functions $\mathcal{H}$, which without loss of generality aim to discriminate the outcome $y=1$ from any other value, denoted here as $y=0$. For any conditional distribution $P(y,\mathbf{x}|\mathbf{s})$ given a group $\mathbf{s}$, let $S \sim P^m$ be a dataset consisting of $m$ examples $(\mathbf{x}_i, y_i )$ sampled i.i.d. from $P(y,\mathbf{x}|\mathbf{s})$. Then for any $0 < \delta < 1$, with probability $1 - \delta$, for every $h \in \mathcal{H}$, we have:
	 \begin{align}
	 &|P(y = 1|\mathbf{s},h)  - P_{S}(y = 1 |\mathbf{s}, h) | \nonumber \\
	 & \ \ \ \ \ \ \ \ \ \leq \tilde{O} \Big ( \sqrt{\frac{\mbox{VCDIM}(\mathcal{H})\log m + \log(1/\delta)}{m}} \Big ) \mbox{ .}
	 \end{align}
\end{theorem}
\begin{proof}
	Let $g(\mathbf{s}') = 1$ when $\mathbf{s}' = \mathbf{s}$ and 0 otherwise, and let $\mathcal{G} = \{ g(\mathbf{s}') \}$.  We see that $\mathcal{G}$ has a VC-dimension of 0.  The result follows directly by applying Theorem \ref{thm:SFgeneralization} (\cite{kearns2018preventing}'s Theorem 2.11) to $\mathcal{H}$ and $\mathcal{G}$, and considering the bound for the distributions $P$ over $(\mathbf{x},y)$ where $P(g(\mathbf{s}') = 1) = 1$.  
\end{proof}
While SF has generalization bounds which depend on the overall number of data points, DF's generalization guarantee requires that we obtain a reasonable number of data points for each intersectional group in order to accurately estimate $\epsilon$-DF.  This difference, the price of removing the minority-biasing term, should be interpreted in the context of the differing goals of our work and \cite{kearns2018preventing}, who aimed to \textbf{prevent fairness gerrymandering} by protecting every conceivable subgroup that could be targeted by an adversary.

In contrast, our goal is to \textbf{uphold intersectionality}, which simply aims to enact a more nuanced understanding of unfairness than with a single protected dimension such as gender or race.  In practice, consideration of 2 or 3 intersecting protected dimensions already improves the nuance of assessment. Sufficient data per intersectional group can often be readily obtained in such cases, e.g. \cite{buolamwini2018gender} studied the intersection of gender and skin color on fairness.  Similarly, \cite{kearns2018preventing} focus on the challenge of \emph{auditing} subgroup fairness when the subgroups cannot easily be enumerated, which is important in the fairness gerrymandering setting.  In contrast, in our intended applications of preserving intersectional fairness the number of intersectional groups is often only around $2^2$ -- $2^5$.

\section{Learning Algorithm}
In this section we introduce a simple, practical learning algorithm for differentially fair classifiers (\emph{DF-Classifiers}).  
Our algorithm uses the fairness cost as a regularizer to balance the trade-off between fairness and accuracy. 
We minimize, with respect to the classifier $M_\mathbf{W}(\mathbf{x})$'s parameters \textbf{W}, a loss function $L_{\mathbf{X}}(\textbf{W})$ plus a penalty on unfairness which is weighted by a tuning parameter $\lambda>0$.  We train fair neural networks using gradient descent (GD) on our objective via backpropagation and automatic differentiation.
The learning objective for training data $\mathbf{X}$ becomes: 
\begin{equation}\label{eq:objective}
    \underset{\textbf{W}}{\text{min}}[L_{\mathbf{X}}(\textbf{W}) + \lambda R_{\mathbf{X}}(\epsilon)]
\end{equation}
where $R_{\mathbf{X}}(\epsilon) =max(0,\epsilon_{M_\mathbf{W}(\mathbf{x})} - \epsilon_1)$ represents the fairness penalty term, and $\epsilon_{M_\mathbf{W}(\mathbf{x})}$ is the $\epsilon$ for $M_\mathbf{W}(\mathbf{x})$. 
To make the objective differentiable, $\epsilon_{M_\mathbf{W}(\mathbf{x})}$ is measured using soft counts (Equation~\ref{eqn:smoothedFairnessSoft}). If $\epsilon_1$ is 0, this penalizes $\epsilon$-DF, and if $\epsilon_1$ is the data's $\epsilon$, this penalizes bias amplification.  Optimizing for bias amplification will also improve $\epsilon$-DF, up to the $\epsilon_1$ threshold. In practice, we found that a warm start optimizing $L_{\mathbf{X}}(\textbf{W})$ only for several ``burn-in'' iterations 
improves convergence. 
For large datasets, stochastic gradient descent (SGD) can be used instead of batch GD. In this case, we recommend that $\epsilon_{M_\mathbf{W}(\mathbf{x})}$ be estimated on a development set $\mathcal{D}$, as minibatch estimates may be unstable in the intersectional data regime.

%
%
%
%
\section{Experiments}
\begin{figure*}[t]
		\centerline{\includegraphics[width=0.9\textwidth]{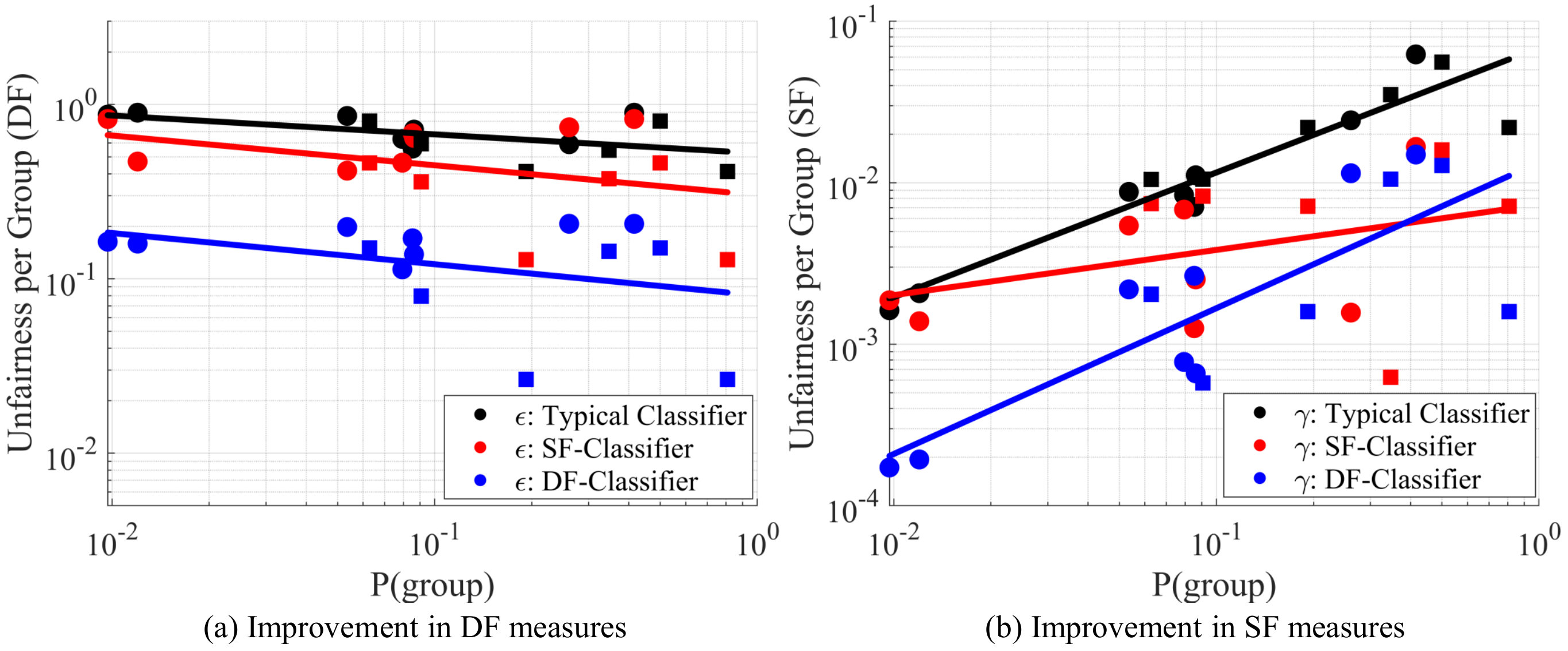}}
		\caption{\small Per-group measurements of (a) $\epsilon$-DF and (b) $\gamma$-SF of the classifiers vs group size (probability), COMPAS dataset, calculated using Equations \ref{eqn:SF} and \ref{eqn:DF} with the group held fixed. Circles: intersectional subgroups. Squares: top-level groups. The methods improve fairness, both per group and overall, but SF-Classifier is empirically seen to ignore minority groups in the overall $\gamma$-SF measurement, calculated as a worst-case over all groups.}
		\label{fig:compareWithSF} 
\end{figure*}

We performed all experiments on two datasets: the Adult 1994 U.S. census income data from the UCI repository \cite{kohavi-nbtree} (protected attributes: \emph{race}, \emph{gender}, USA vs non-USA \emph{nationality}), and the COMPAS dataset regarding a system that is used to predict criminal recidivism \cite{angwin2016machine} (protected attributes: \emph{race} and \emph{gender}).\footnote{Predicted income, used for consequential decisions like housing approval, may result in \emph{digital redlining} \cite{barocas2016big}.}

\subsection{Fair Learning Algorithm}
The goals of our experiments were to demonstrate the practicality of our \emph{DF-Classifier} method in learning an intersectionally fair classifier, and to compare its behavior to a learned subgroup fair \emph{SF-Classifier} and a typical classifier (without the fairness penalty term of Equation~\ref{eq:objective}), especially with regards to minorities.  Instead of \cite{kearns2018preventing}'s algorithm, we trained the SF-Classifier using the same GD+backpropagation approach, replacing $\epsilon$ with $\gamma$ in Equation \ref{eq:objective}, i.e. $R_{\mathbf{X}}(\gamma) =max(0,\gamma_{M_\mathbf{W}(\mathbf{x})} - \gamma_1)$. This simplifies and speeds up learning to handle deep neural networks.

All classifiers were trained on a common neural network architecture via adaptive gradient descent optimization (Adam)  with learning rate = $0.01$ using pyTorch. The configuration of the neural network was $3$ hidden layers, $16$ neurons in each layer, ``relu" and ``sigmoid" activations for the hidden and output layers, respectively. We trained for $500$ iterations, disabling the fairness penalties for the first $50$ ``burn-in'' iterations. 
We chose $\lambda$ as $0.1$ and $1.0$ for \emph{DF-Classifier} and \emph{SF-Classifier}, respectively, as a best trade-off value via grid search over the randomly held out $20\%$ development sets. 

We learned fair classifiers in several settings: 1) we set the target thresholds to perfect fairness, $\epsilon_{1}$=$0.0$ and $\gamma_{1}$=$0.0$ for DF-Classifier and SF-Classifier, respectively, and 2) to penalize bias amplification by the algorithm, by setting the thresholds to $\epsilon_{1}$=$\epsilon_{data}$ and $\gamma_{1}$=$\gamma_{data}$ for DF-Classifier and SF-Classifier, respectively. Finally, to protect the $80\%$-rule we set $\epsilon_{1}$=$-\log 0.8 = 0.2231$ for DF-Classifier only. Since there is no straightforward way to enforce the $80\%$-rule for SF-Classifier, it was not considered in this analysis. 

Tables~\ref{fig:Adult-dataset} and ~\ref{fig:COMPAS-dataset} compare the classifiers on the Adult and COMPAS datasets, respectively. Both DF-Classifier and SF-Classifier were able to \textbf{substantially improve their fairness metrics over the typical classifier, with modest costs in accuracy, F1 score, and ROC AUC}, and the trade-off varied roughly monotonically in the target value $\epsilon_1$ or $\gamma_1$.  Based on \emph{soft count} estimation (Equation \ref{eqn:smoothedFairnessSoft}), the DF-Classifier with $\epsilon_1 = 0$ improved from $\epsilon=1.646$ to $\epsilon=0.428$ on Adult with a loss of $2.8$ percentage points of accuracy. On COMPAS, it improved from $\epsilon=0.773$ to $\epsilon=0.180$, corresponding to a worst-case difference in utility between groups of a factor of $e^\epsilon \approx 1.2$, with a loss of just 1.4 percentage points of accuracy.  When trained to prevent \textbf{bias amplification}, the fairness metrics were improved with \textbf{little} (COMPAS) \textbf{to no} (Adult) \textbf{reduction in accuracy}.  While SF-Classifier typically had slightly higher accuracy under the same settings, \textbf{DF-Classifier often greatly improved $\gamma$-SF as well, while SF-Classifier enjoyed only modest improvements in $\epsilon$-DF}.  The conclusions were similar with ``hard count'' smoothed EDF estimates (Equation \ref{eqn:smoothedFairness}), but the metrics' estimates were higher.


An important goal of this work was to consider the impact of the fairness methods on minority groups. 
In Figure \ref{fig:compareWithSF}, we report the ``per-group unfairness,'' defined as Equations \ref{eqn:SF} and \ref{eqn:DF} with one group held fixed, versus the group's probability (i.e. size) on the COMPAS dataset. Both methods improve their corresponding per-group unfairness measures over the typical classifier. 
On the other hand, similarly to Figure \ref{fig:compareWithSFdata}, the $\gamma$-SF metric only assigns high per-group unfairness values to large groups in its measurement, so \textbf{minority groups are not able to influence the overall $\gamma$-SF unfairness}.  This was not the case \textbf{for $\epsilon$-DF metric, where groups of various sizes had similarly high per-group $\epsilon$ values}.  Furthermore, the \textbf{DF-Classifier improved the per-group fairness under both metrics for groups of all sizes, while the SF-classifier did not improve the per-group $\gamma$-SF for small groups}.  
%
Our overall conclusion is that the \emph{DF-Classifier is able to achieve intersectionally fair classification with minor loss in performance, while providing greater protection to minority groups than when enforcing subgroup fairness}.

\begin{table}[t]
\centering
\resizebox{0.4\textwidth}{!}{
\begin{tabular}{lcccc}
\hline
        & \multicolumn{4}{c}{Gini Coefficient ($G$)}                                        \\ 
Dataset & $\epsilon_{Data}$ & $\gamma_{Data}$ & $\epsilon_{LR}$ & $\gamma_{LR}$ \\ \hline
Adult   & \textbf{0.099}               & 0.256              & \textbf{0.126}              & 0.257            \\ 
COMPAS  & \textbf{0.151}                & 0.376              & \textbf{0.135}              & 0.343            \\ \hline
\end{tabular}
}
\textbf{\caption{\small Comparison of the inequity in the per-group allocation of the $\epsilon$-DF and $\gamma$-SF metrics via the Gini coefficient 
(lower is better).}
\label{fig:gini-table}}
\end{table}

\begin{table}[t]
\centering
\resizebox{0.4\textwidth}{!}{
\begin{tabular}{lcc}
\multicolumn{3}{c}{COMPAS Dataset}                      \\ \hline
Protected attributes      & $\epsilon$-DF & $\gamma$-SF \\ \hline
race                      & 0.1003         & 0.0070       \\
gender                    & 0.9255         & \textbf{\textcolor{red}{0.0656}}       \\
race, gender              & 1.3156         & 0.0604       \\ \hline
\multicolumn{3}{l}{}                                    \\
\multicolumn{3}{c}{Adult Dataset}                       \\ \hline
Protected attributes      & $\epsilon$-DF & $\gamma$-SF \\ \hline
nationality               & 0.2177         & 0.0045       \\
race                      & 0.9188         & 0.0128       \\
gender                    & 1.0266         & \textbf{\textcolor{red}{0.0434}}      \\
gender, nationality       & 1.1511         & 0.0431       \\
race, nationality         & 1.1534         & 0.0163       \\
race, gender              & 1.7511         & 0.0451       \\
race, gender, nationality & 1.9751         & 0.0455       \\ \hline
\end{tabular}
}
\textbf{\caption{\small Protection of intersectionality by DF metric on COMPAS and Adult dataset. The cases in red are where $\gamma$-SF violates the intersectionality property enjoyed by $\epsilon$-DF (Theorem \ref{thm:fullIntersectionality}).}
\label{fig:adult-intersectionality}}
\end{table}

\subsection{Inequity of Fairness Measures}
We have seen that the $\gamma$-SF metric downweights the consideration of minorities (cf. Figures \ref{fig:compareWithSFdata} and \ref{fig:compareWithSF}). In this experiment, we quantify the resulting inequity of fairness consideration using the \emph{Gini coefficient}~\cite{lorenz1905methods}, a commonly used measure of statistical dispersion which is often used to represent the inequity of income. 
The \emph{Gini coefficient ($G$) of a fairness metric} $F$ is calculated as 
\begin{equation}
G = \frac{1}{2\mu}\sum_{i=1}^{n}\sum_{j=1}^{n} P(\mathbf{s}_i)P(\mathbf{s}_j)|F_{\mathbf{s}_i}-F_{\mathbf{s}_j}|\mbox{ ,}
\label{eqn:gini}
\end{equation}
where $\mu = \sum_{i=1}^{n}F_{\mathbf{s}_i}P(\mathbf{s}_i)$ and $P(\mathbf{s}_i)$ is the fraction of population belonging to the $i^{th}$ intersectional group, while $F_{\mathbf{s}_i}$ represents the fairness measure (i.e. per-group $\epsilon$ or $\gamma$) of that group.  For a fixed algorithm and data distribution, a fairness metric with a smaller Gini coefficient distributes its (un)fairness consideration more equitably across the population, which is typically desirable in the sense that \emph{the entire population has a voice in the determination of (un)fairness}.

Table~\ref{fig:gini-table} shows a comparison of $G$ values for the $\epsilon$-DF and $\gamma$-SF metrics on the Adult and COMPAS datasets. Both fairness metrics are measured for the labeled dataset (i.e. $\epsilon_{Data}$) as well as for a logistic regression (LR) classifier (i.e. $\epsilon_{LR}$) trained on the same dataset. In all the experiments, the $G$ value for $\epsilon$-DF is much lower compared to $\gamma$-SF's $G$ value. 
Thus, $\epsilon$-DF was observed to provide a more equitable distribution of its per-group fairness measurements, presumably due to its more inclusive treatment of minority groups.

\subsection{Evaluation of Intersectionality Property}
In our final experiment (Table~\ref{fig:adult-intersectionality}), we studied the ability of $\gamma$-SF to preserve the intersectionality property shown for $\epsilon$-DF in Theorem~\ref{thm:fullIntersectionality}, by measuring fairness with different sets of protected attributes. The property is violated if removing a protected attribute increases the metric.  As expected, $\epsilon$-DF obeyed the intersectionality property, but $\gamma$-SF violated it as $\gamma$ for \emph{gender} $>$ $\gamma$ for \emph{race} $\times$ \emph{gender} (COMPAS), and $\gamma$ for \emph{gender} $>$ $\gamma$ for \emph{gender} $\times$ \emph{nationality} (Adult).

\section{Conclusion}
We introduced three AI fairness definitions which satisfy intersectional fairness desiderata, \emph{differential fairness} and its \emph{bias amplification} and confounder-aware counterparts, and proved their attractive properties regarding the law, privacy, economics, and statistical learning, along with a learning algorithm to enforce them.  With extensive experiments across two datasets, we have shown that our criteria can be practically attained, and they behave more equitably with regard to minority groups than subgroup fairness.  In future work, we plan to investigate the impact of data sparsity on the measurement and enforcement of fairness in the intersectional multi-attribute regime.



\section*{Acknowledgment}
%
We thank Rosie Kar for valuable advice and feedback regarding intersectional feminism.

%
\bibliographystyle{plain}
\bibliography{references}

%
\appendix


\subsection{Proof of Lemma \ref{lem:rewriteDF}}
\begin{proof}
The definition of $\epsilon$-differential fairness is, for any $\theta \in \Theta$, $y \in \mbox{Range}(M)$, $(\mathbf{s}_i, \mathbf{s}_j) \in A \times A$ where $P(\mathbf{s}_i|\theta) > 0$, $P(\mathbf{s}_j|\theta) > 0$,
\begin{align}
e^{-\epsilon} \leq \frac{P_{M, \theta}(M(\mathbf{x}) = y|\mathbf{s}_i, \theta)}{P_{M, \theta}(M(\mathbf{x}) = y|\mathbf{s}_j, \theta)} &\leq e^\epsilon  \mbox{ .}
\end{align}
Taking the log, we can rewrite this as:
\begin{align}
-&\epsilon \leq \log P_{M, \theta}(M(\mathbf{x}) = y|\mathbf{s}_i, \theta) \nonumber \\
&- \log P_{M, \theta}(M(\mathbf{x}) = y|\mathbf{s}_j, \theta) \leq \epsilon \mbox{ .}
\end{align}
The two inequalities can be simplified to:
\begin{align}
| \log P_{M, \theta}(M(\mathbf{x}) = y|\mathbf{s}_i, \theta) - \log P_{M, \theta}(M(\mathbf{x}) = y|\mathbf{s}_j, \theta) | &\leq \epsilon  \mbox{ .}
\end{align}
For any fixed $\theta$ and $y$, we can bound the left hand side by plugging in the worst case over $(\mathbf{s}_i, \mathbf{s}_j)$, 
\begin{align}
&| \log P_{M, \theta}(M(\mathbf{x}) = y|\mathbf{s}_i, \theta) - \log P_{M, \theta}(M(\mathbf{x}) = y|\mathbf{s}_j, \theta) | \nonumber \\
&\leq \log \max_{\mathbf{s}: P(\mathbf{s}|\theta) > 0} P_{M, \theta}(M(\mathbf{x}) = y|\mathbf{s}, \theta) \nonumber \\
&\ \ \ \ - \log \min_{\mathbf{s}: P(\mathbf{s}|\theta) > 0} P_{M, \theta}(M(\mathbf{x}) = y|\mathbf{s}, \theta) \mbox{ .}
\end{align}
Plugging in this bound, which is achievable and hence is tight, the criterion is then equivalent to:
\begin{align}
& \log \max_{\mathbf{s}: P(\mathbf{s}|\theta) > 0} P_{M, \theta}(M(\mathbf{x}) = y|\mathbf{s}, \theta) \nonumber \\
&\ \ \ \ - \log \min_{\mathbf{s}: P(\mathbf{s}|\theta) > 0} P_{M, \theta}(M(\mathbf{x}) = y|\mathbf{s}, \theta) \leq \epsilon  \mbox{ .}
\end{align}
\end{proof}

\subsection{Proof of Theorem \ref{thm:confounders}}
\begin{proof}
	Let $\theta \in \Theta$, $y \in \mbox{Range}(M)$, $c \in C$, and $(\mathbf{s}_i, \mathbf{s}_j) \in A \times A$ where $P(\mathbf{s}_i|\theta) > 0$ and $P(\mathbf{s}_j|\theta) > 0$. We have:
	\begin{align}
	&\frac{P_{M, \theta}(M(x) = y|\mathbf{s}_i, \theta)}{P_{M, \theta}(M(x) = y|\mathbf{s}_j, \theta)} \nonumber \\
	&= \frac{\sum_{c \in C} P_{M, \theta}(M(x) = y|\mathbf{s}_i, c, \theta) P_{M,\theta}(c|\mathbf{s}_i,\theta)}{\sum_{c \in C} P_{M, \theta}(M(x) = y|\mathbf{s}_j, c, \theta) P_{M,\theta}(c|\mathbf{s}_j,\theta)} \nonumber \\
	&= \frac{\sum_{c \in C} \frac{P_{M, \theta}(M(x) = y|\mathbf{s}_i, c, \theta)}{P_{M, \theta}(M(x) = y|\mathbf{s}_j, c, \theta)} P_{\theta}(c|\mathbf{s}_i,\theta)}{\sum_{c \in C} \frac{P_{M, \theta}(M(x) = y|\mathbf{s}_j, c, \theta)}{P_{M, \theta}(M(x) = y|\mathbf{s}_j, c, \theta)} P_{\theta}(c|\mathbf{s}_j,\theta)} \nonumber \\
	&= \sum_{c \in C} \frac{P_{M, \theta}(M(x) = y|\mathbf{s}_i, c, \theta)}{P_{M, \theta}(M(x) = y|\mathbf{s}_j, c, \theta)} P_{\theta}(c|\mathbf{s}_i,\theta)\nonumber \\
	& \leq \sum_{c \in C} e^\epsilon P_{\theta}(c|\mathbf{s}_i,\theta)= e^{\epsilon} \mbox{ .}
	\end{align}
	Reversing $\mathbf{s}_i$ and $\mathbf{s}_j$ and taking the reciprocal shows the other inequality.
\end{proof}

\subsection{Related Work}
\label{sec:related}
This section discusses relationships with other concepts in fairness, privacy, and in the treatment of subsets of protected groups.  

\subsubsection{Fairness Definitions}
An overview of fairness research can be found in \cite{berk2017fairness}.  We briefly describe several of the most influential mathematical definitions of fairness below, and discuss their relationships to our proposed differential fairness criterion.

\textbf{The 80\% rule:} Our criterion is related to the \emph{80\% rule}, a.k.a. the \emph{four-fifths rule}, a guideline for identifying unintentional discrimination in a legal setting which identifies disparate impact in cases where $P(y=1|s_1)/P(y=1|s_2) \leq 0.8$, for a favourable outcome $y=1$, disadvantaged group $s_1$, and best performing group $s_2$ \cite{eeoc1966guidelines}.  This corresponds to testing that $\epsilon \geq -\log 0.8 = 0.2231$, in a version of Equation \ref{eqn:DF} where only the outcome $y=1$ is considered.

\textbf{Demographic Parity:}  \cite{dwork2012fairness} defined (and criticized) the fairness notion of \emph{demographic parity}, a.k.a. \emph{statistical parity}, which requires that $P(y|\mathbf{s}_i) = P(y|\mathbf{s}_j)$ for any outcome $y$ and pairs of protected attribute values $\mathbf{s}_i$, $\mathbf{s}_j$ (here assumed to be a single attribute).  This can be relaxed, e.g. by requiring the total variation distance between the distributions to be less than $\epsilon$.  Differential fairness is closely related as it  also aims to match probabilities of outcomes, but measures differences using ratios, and allows for multiple protected attributes.  The criticisms of \cite{dwork2012fairness} are mainly related to ways in which subgroups of the protected groups can be treated differently while maintaining demographic parity, which they call ``\emph{subset targeting},'' and which \cite{kearns2018preventing} term ``\emph{fairness gerrymandering}.''  Differential fairness explicitly protects the intersection of multiple protected attributes, which can be used to mitigate some of these abuses. 

\textbf{Equalized Odds:} To address some of the limitations with demographic parity, \cite{hardt2016equality} propose to instead ensure that a classifier has equal error rates for each protected group.  This fairness definition, called \emph{equalized odds}, can loosely be understood as a notion of ``demographic parity for error rates instead of outcomes.''  Unlike demographic parity, equalized odds rewards accurate classification, and penalizes systems only performing well on the majority group.  
However, theoretical work has shown that equalized odds is typically incompatible with correctly calibrated probability estimates \cite{pleiss2017fairness}.  It is also a relatively weak notion of fairness from a civil rights perspective compared to demographic parity, as it does not ensure that outcomes are distributed equitably.  Hardt et al. also propose a variant definition called \emph{equality of opportunity}, which relaxes equalized odds to only apply to a ``deserving'' outcome.  It is straightforward to extend differential fairness to a definition analogous to equalized odds, although we leave the exploration of this for future work. A more recent algorithm for enforcing equalized odds and equality of opportunity for kernel methods was proposed by \cite{donini2018empirical}.


\textbf{Individual Fairness (``Fairness Through Awareness''):}  The \emph{individual fairness} definition, due to \cite{dwork2012fairness}, mathematically enforces the principle that \emph{similar individuals should get similar outcomes} under a classification algorithm.  An advantage of this approach is that it preserves the privacy of the individuals, which can be important when the user of the classifications (the \emph{vendor}), e.g. a banking corporation, cannot be trusted to act in a fair manner.  However, this is difficult to implement in practice as one must define ``similar'' in a fair way.  The individual fairness property also does not necessarily generalize beyond training set.  In this work, we take inspiration from Dwork et al.'s \emph{untrusted vendor} scenario, and the use of a privacy-preserving fairness definition to address it.

\textbf{Counterfactual Fairness: } \cite{kusner2017counterfactual} propose a causal definition of fairness.  Under their \emph{counterfactual fairness} definition, changing protected attributes $A$, while holding things which are not causally dependent on $A$ constant, will not change the predicted distribution of outcomes.  While theoretically appealing, there are difficulties in implementing this in practice.  First, it requires an accurate causal model at the fine-grained individual level, while even obtaining a correct population-level causal model is generally very difficult.  To implement it, we must solve a challenging causal inference problem over unobserved variables, which generally requires approximate inference algorithms. (In the case of differential fairness, we advocate the use of Bayesian models which typically require approximate inference as well, although empirical distributions can be used if sufficient data is available.) Finally, to achieve counterfactual fairness, the predictions (usually) cannot make direct use of any descendant of $A$ in the causal model.  This generally precludes using \emph{any of the observed features} as inputs.

\textbf{Threshold Tests:} \cite{simoiu2017problem} address \emph{infra-marginality} by modeling risk probabilities for different subsets (i.e. attribute values) within each protected category, and requiring algorithms to threshold these probabilities at the same points when determining outcomes.  In contrast, based on \emph{intersectionality} theory, our proposed differential fairness criterion specifies protected categories whose intersecting subsets should be treated equally, regardless of differences in risk across the subsets.  Our definition is appropriate when the differences in risk are due to structural systems of oppression, i.e. the risk probabilities themselves are impacted by an unfair process.  We also provide a \emph{bias amplification} version of our metric, following \cite{zhao2017men}, which is more in line with the infra-marginality perspective.

\subsubsection{Privacy Definitions}

\textbf{Differential Privacy:} Our work on fairness is inspired by \emph{differential privacy}, the gold-standard notion of privacy for data-driven algorithms \cite{dwork2013algorithmic}.  Essentially, differential privacy is a promise: if an individual contributes their data to a dataset, their resulting utility, due to algorithms applied to that dataset, will not be substantially affected.  The privacy guarantee is obtained via the use of randomized algorithms, typically by adding sufficient noise, e.g. from the Laplace distribution, in order to obfuscate the impact of any one data point on the algorithms' outputs.

\begin{definition} $M(\mathbf{x})$ is $\epsilon$-differentially private if	
	\begin{equation*}
	\frac{P(M(\mathbf{x}) \in \mathcal{S})}{P(M(\mathbf{x}') \in \mathcal{S})} \leq e^\epsilon 
	\end{equation*}
	for all outcomes $\mathcal{S}$, and pairs of databases $\mathbf{x}$, $\mathbf{x}'$ differing in a single element.
\end{definition}
Similarly to differential privacy, our proposed differential fairness definition bounds ratios of probabilities of outcomes resulting from a mechanism.  However, there are several important differences.  When bounding these ratios, differential fairness considers different values of a set of protected attributes, rather than databases that differ in a single element.  It posits a specified set of possible distributions which may generate the data, while differential privacy implicitly assumes that the data are independent \cite{kifer2014pufferfish}.  Finally, since differential fairness considers randomness in data as well as in the mechanism, it can be satisfied with a deterministic mechanism, while differential privacy can only be satisfied with a randomized mechanism.


\subsubsection{Other Related Work}

\textbf{Fairness and Intersectionality:} Of particular relevance to this work, fairness in an intersectional setting has been considered by \cite{buolamwini2018gender} in a computer vision context, and by \cite{kearns2018preventing} and \cite{hebert-johnson2018multicalibration}, who aim to protect certain subgroups by preventing ``fairness gerrymandering.''

\textbf{Fairness and Uncertainty:} Bayesian modeling of fairness has been performed by \cite{simoiu2017problem} in the context of stop-and-frisk policing, and by \cite{kusner2017counterfactual}, who use Bayesian inference on a causal model.  As an alternative to the Bayesian methodology, adversarial methods are another strategy for managing uncertainty in a fairness context, e.g. \cite{beutel2017data} apply this approach to the setting of ensuring fairness given a limited number of observations in which demographic information is available.  \cite{roth2006modeling} 
study various hypothesis testing methods for the 80\% rule in the small data regime.

\textbf{Fairness and Privacy:}  The work of \cite{jagielski2018differentially} also addresses untrusted vendors, focusing on differentially private fair learning algorithms (with respect to protected attributes) which obtain obtain fairness under a different criterion.  In contrast, differential fairness ensures that the behavior of the \emph{final algorithm}, rather than the learning process for the algorithm, preserves the privacy of the individuals' protected attributes.

\end{document}